\documentclass[smallextended,referee,envcountsect,]{svjour3}
\smartqed
\usepackage{graphicx}
\journalname{JOTA}

\usepackage{hyperref}
\usepackage{url}

\usepackage[utf8]{inputenc} 
\usepackage[T1]{fontenc}    
\usepackage{url}            
\usepackage{booktabs}       
\usepackage{amsfonts}       
\usepackage{nicefrac}       
\usepackage{microtype}      
\usepackage{dsfont,mathrsfs}
\usepackage{graphicx}
\usepackage{amsmath,amssymb}
\usepackage{mathtools}
\usepackage{dsfont,mathrsfs}
\usepackage{enumerate}
\usepackage{float}
\usepackage{enumitem}
\usepackage[ruled,vlined]{algorithm2e}
\usepackage{subfigure}
\usepackage{algorithmic}
\usepackage{comment}

\usepackage[textwidth=3cm,textsize=footnotesize]{todonotes}

\DeclareMathOperator*{\argmax}{argmax}

\newcommand{\wrt}{w.r.t.\@\xspace}

\newcommand{\block}[3]{\left#1 #2 \right#3} 
\newcommand{\sqb}[1]{\block{[}{#1}{]}} 
\newcommand{\ip}[1]{\block{\langle}{#1}{\rangle}} 
\newcommand{\tup}[1]{\ip{#1}} 

\newcommand{\EE}[1]{\mathsf{E}\sqb{#1}} 

\newcommand{\sn}[1]{\mathcal{#1}} 
\newcommand{\fun}[3]{#1 : #2 \rightarrow #3} 
\newcommand{\eqdef}{\doteq} 
\newcommand{\tkern}[1]{\sn{#1}} 

\newcommand{\Pkern}{\tkern{P}}

\renewcommand{\epsilon}{\varepsilon}

\def\ConstC{\operatorname{C}}
\def\smallx{\mathsf{x}}
\def\rmd{\mathrm{d}}
\def\Xset{\mathsf{X}}
\def\Aset{\mathsf{A}}
\def\Xsetn{\mathsf{X}_N}
\def\rset{\mathbb{R}}
\def\dx{{d_\Xset}}
\newcommand{\PP}{\mathrm{P}}
\def\rme{\mathrm{e}}
\def\MK{{\rm P}}
\def\nset{\mathbb{N}}

\def\Mk{{\mathcal M}}
\def\diam{\operatorname{diam}}
\newcommand{\PE}{\mathsf{E}}
\newcommand{\indiacc}[1]{\mathbb{I}_{\{#1\}}}
\def\Var{\operatorname{Var}}
\def\xiv{\boldsymbol{\xi}}

\def\opT{\mathsf{T}}
\def\opM{\mathsf{M}}

\newtheorem{assum}{A\hspace{-2pt}}
\usepackage{cleveref}
\usepackage{hyperref}

\crefname{assum}{A\hspace{-2pt}}{A\hspace{-2pt}}

\newtheorem{thm}{Theorem}[section]
\newtheorem{rem}{Remark}[section]
\newtheorem{prop}{Proposition}[section]

\newtheorem{cor}{Corollary}[section]
\newtheorem{defi}{Definition}[section]

\begin{document}

\title{UVIP: Model-Free Approach to Evaluate Reinforcement Learning Algorithms}

\author{Denis Belomestny \and
        Ilya Levin \and
        Alexey Naumov  \and
        Sergey Samsonov}

\institute{Denis Belomestny \at
           Duisburg-Essen University
           and HSE University \\
          denis.belomestny@uni-due.de
          \and
          Ilya Levin, Corresponding author \at
             HSE University\\
             ivlevin@hse.ru
        \and
              Alexey Naumov \at
              HSE University and Steklov Mathematical Institute of \\
              Russian Academy of Sciences \\
              anaumov@hse.ru
          \and
          Sergey Samsonov \at
           HSE University\\
            svsamsonov@hse.ru
}

\date{Received: date / Accepted: date}

\maketitle

\begin{abstract}
Policy evaluation is an important instrument for the comparison of different algorithms in Reinforcement Learning (RL). However, even a precise knowledge of the value function $V^{\pi}$ corresponding to a policy $\pi$ does not provide reliable information on how far the policy $\pi$ is from the optimal one. We present a novel model-free upper value iteration procedure ({\sf UVIP}) that allows us to estimate the suboptimality gap $V^{\star}(x) - V^{\pi}(x)$ from above and to construct confidence intervals for \(V^\star\). Our approach relies on upper bounds to the solution of the Bellman optimality equation via the martingale approach. We provide theoretical guarantees for {\sf UVIP} under general assumptions and illustrate its performance on a number of benchmark RL problems. \\ \\ Communicated by Alexander Vladimirovich Gasnikov.
\end{abstract}


\section{Introduction}
The key objective of Reinforcement Learning (RL) is to learn an optimal agent's behavior in an unknown environment. A natural performance metric is given by the value function $V^{\pi}$, which is the expected total reward of the agent following $\pi$. There are efficient algorithms to evaluate this quantity, e.g., temporal difference methods \cite{sutton_1988_td}, \cite{tsitsiklis:td:1997}. Unfortunately, even a precise knowledge of $V^{\pi}$ does not provide reliable information on how far the policy $\pi$ is from the optimal one. At the same time, practitioners are often interested in quantitative guarantees on the \emph{suboptimality gap (policy error)} $\Delta_\pi(x) \eqdef  V^\star(x)-V^\pi(x)$ or, more generally, in tight confidence bounds for the optimal value function $V^\star$. To address this issue, a popular quality measure is the \emph{regret} of the algorithm, which is the difference between the total sum of rewards accumulated when following the optimal policy and the sum of rewards obtained when following the current policy $\pi$ (see, e.g., \cite{JMLR:v11:jaksch10a}). In the setting of finite state- and action-space Markov Decision Processes (MDP), there is a variety of regret bounds for popular RL algorithms like {\sf Q}-learning \cite{q_learning_efficient}, optimistic value iteration \cite{azar:2017}, and many others. Unfortunately, regret bounds beyond the discrete setup are much less common in the literature. An even more crucial drawback of the regret-based comparison is that regret bounds are typically pessimistic and rely on the unknown quantities of the underlying MDPs.
\par 
In this paper, we address the problem of estimating $\Delta_\pi(x)$ by constructing model-free upper confidence bounds for the optimal value function $V^\star$ and, consequently, for the policy error $V^\star - V^\pi$. Our starting point is the Bellman optimality equation, which characterizes $V^\star$ as a fixed point of the Bellman operator. Rather than trying to approximate $V^\star$ directly, we introduce the notion of an \emph{upper solution} to the Bellman equation. By combining this idea with the martingale-based duality argument, we design an upper value iteration procedure (UVIP) which, given an arbitrary policy $\pi$, produces an almost-sure upper bound on $V^\pi$ and thus on $V^\star$ using only samples from the (unknown) transition kernel.

\vspace{-0.4cm}
\paragraph{Contributions and Organization} The contributions of this paper are three-fold:
\vspace{-0.3cm}
\begin{itemize}[leftmargin=5mm,noitemsep]
\item We propose a novel approach to construct model-free confidence bounds for the optimal value function \(V^\star\) based on a notion of upper solutions.
\item Given a policy \(\pi,\) we propose an upper value iterative procedure ({\sf UVIP}) for constructing an (almost sure) upper bound for \(V^{\pi}\) such that it coincides with \(V^\star\) if \(\pi=\pi^\star.\)
\item We study convergence properties of the approximate {\sf UVIP} in the case of general state and action spaces. In particular, we show that the variance of the resulting upper bound is small if $\pi$ is close to $\pi^*$, leading to tight confidence bounds for \(V^\star\).
\end{itemize}
The paper is organized as follows. First, in Section~\ref{sec:preliminary}, we briefly recall the main concepts related to MDPs and introduce some notations. Then, in Section~\ref{sec:lit_review}, we discuss the contributions related to our paper. In Sections~\ref{sec: alg th} and \ref{sec:fit-vup}, we introduce the framework of {\sf UVIP} and discuss its basic properties. In Section~\ref{sec: theory}, we perform a theoretical study of the approximate {\sf UVIP}. The numerical results are collected in Section~\ref{sec:num}. Section~\ref{sec:conclusion} concludes the paper. Section~\ref{sec:proof_of_main_res} in the appendix is devoted to the proof of the main theoretical results. 

\vspace{-0.3cm}
\paragraph{Notations and definitions.}
For $N \in \nset$, we define $[N] \eqdef \{1, \ldots, N\}$. Let us denote the space of bounded measurable functions with domain \(\Xset\) by \(\sn{B}(\Xset )\), equipped with the norm \(\|f\|_{\Xset}=\sup_{x\in \Xset} |f(x)|\) for any \(f\in \sn{B}(\Xset )\). In what follows, whenever a norm is uniquely identifiable from its argument, we will drop the index of the norm denoting the underlying space. We denote by $\MK^a$
an $\sn{B}(\Xset ) \to \sn{B}(\Xset)$ 
operator defined by $(\MK^a V )(x)= \int V(x') \MK^a(dx'|x)$. For an arbitrary metric space $(\sn{X}, \rho_{\sn{X}})$ and function $f: \sn{X} \to \rset$, we define $\mathrm{Lip}_{\rho_{\sn{X}}}(f) \eqdef \sup_{x \neq y} |f(x) - f(y)|/\rho_{\sn{X}}(x,y)$.

\vspace{-3ex}
\section{Preliminary}
\label{sec:preliminary}
A \emph{Markov Decision Process } (MDP) is a tuple $(\Xset, \Aset, \Pkern, r)$,
 where $\Xset$ is the state space, $\Aset$ is the action space,
$\Pkern= (\MK^a)_{a \in \Aset}$ is the \emph{transition probability kernel}, and
$r = (r^a)_{a \in \Aset}$ is the \emph{reward function}.
For each state $x\in \Xset$ and action $a\in \Aset$, $\MK^a(\cdot | x)$ stands
for a distribution over the states in $\Xset$,
that is, the distribution over the next states given that action $a$ is taken in the state $x$.
For each action $a\in \Aset$ and state $x\in \Xset$, $r^a(x)$ gives a reward received when action $a$ is taken in state $x$.
An MDP describes the interaction of an agent and its environment. When an action $A_t\in \Aset$ at time \(t\) is chosen by the agent, the state \(X_t\) transitions to $X_{t+1}\sim \MK^{A_t}(\cdot|X_t)$.
The agent's goal is to maximize
the \emph{expected total discounted reward}, $\PE[\sum_{t=0}^\infty \gamma^t r^{A_t}(X_t)]$,
where $0 < \gamma < 1$ is the \emph{discount factor}.
A rule describing the way an agent acts given its past actions and observations is called a \emph{policy}.
The \emph{value function} of a policy $\pi$ in a state $x \in \Xset$, denoted by $V^\pi(x)$, is $V^{\pi}(x) = \PE[\sum_{t=0}^\infty \gamma^t r^{A_t}(X_t)\vert X_0 = x]$, that is, the expected total discounted reward when the initial state is $X_0 = x$, assuming the agent follows the policy $\pi$. Similarly, we define the action-value function $Q^{\pi}(x,a) = \PE[\sum_{t=0}^\infty \gamma^t r^{A_t}(X_t)\vert X_0 = x, A_0 = a]$. An \emph{optimal policy} is one that achieves the maximum possible value amongst all policies in each state $x\in \Xset$. The \emph{optimal value} for state $x$ is denoted by $V^\star(x)$. 
A \emph{deterministic Markov policy} can be identified with a map $\pi: \Xset \to \Aset$, and the space of measurable deterministic Markov policies will be denoted by $\Pi$.
When, in addition, the reward function is bounded, which we assume from now on,
all the value functions are bounded, and
one can always find a deterministic Markov policy that is optimal \cite{puterman2014markov}. We also define a \emph{greedy policy} w.r.t. the action-value function $Q(x,a)$, which is a deterministic policy $\pi(x) \in \argmax_{a \in \Aset} Q(x,a)$. The \emph{Bellman return operator} \wrt{} $\MK$, $\opT_{\MK}: \sn{B}(\Xset) \to \sn{B}(\Xset\times\Aset)$, is defined by
$(\opT_{\MK}V)(x,a) = r^a(x) + \gamma \MK^a V(x)$,
and the \emph{maximum selection} operator $\fun{\opM}{\sn{B}(\Xset\times\Aset)}{\sn{B}(\Xset)}$ is defined by $(\opM V^{\cdot})(x) = \max_a V^a(x)$.
Then $\opM \opT_{\MK}$ corresponds to the \emph{Bellman optimality operator}; see \cite{puterman2014markov}.
The optimal value function $V^\star$ satisfies a non-linear fixed-point equation 
\begin{equation}
\label{eq:Bellman_optimality_equation}
V^\star(x) = \opM \opT_{\MK}V^\star(x),
\end{equation}
which is known as the \emph{Bellman optimality equation}. We write $Y^{x,a}$, $x \in \Xset$, $a\in \Aset$, for a random variable generated according to $\MK^a(\cdot | x)$ and define a random Bellman operator \((\widetilde{\opT}_{\MK}V)(x)\mapsto  r^a(x) + \gamma V(Y^{x,a})\). We say that a (deterministic) policy \(\pi\) is \emph{greedy} \wrt{} a function \(V\in \sn{B}(\Xset)\) if, for all \(x\in \Xset\),
\[
\textstyle{
\pi(x)\in \argmax_{a\in \Aset} \left\{r^a(x)+\gamma\MK^a V(x)\right\}.}
\]

\vspace{-3ex}
\section{Related Works}
\label{sec:lit_review}
There is a large body of work on theoretical guarantees for $\Delta_\pi(x)$ in approximate dynamic programming and model-based RL, including results on fitted Value- and $Q$-iteration and on policy error bounds for model-based approaches with factored linear models; see, for example, \cite{antos2007fitted}, \cite{szepesvari2010algorithms}, \cite{pires2016policy} and references therein. These bounds typically depend on the problem characteristics, which are not known in practice. Moreover, they are often tied to the specific algorithm that produced~$\pi$ and are not directly applicable as a generic evaluation tool for an arbitrary policy. For instance, in Approximate Policy Iteration ({\sf API}, \cite{Bertsekas+Tsitsiklis:1996}), all existing bounds for \(\Delta_\pi(x)\) depend on the one-step error induced by the approximation of the $Q$-function. This one-step error is difficult to quantify since it depends on the unknown smoothness properties of the $Q$-function. Similarly, in policy gradient methods (see, e.g., \cite{sutton:book:2018}), there is an approximation error due to the choice of the family of policies that can be hardly quantified.
\par 
The approach based on the policy optimism principle (see \cite{efroni2019tight}) suggests to initialise the value iteration algorithm using an upper bound (optimistic value) for $V^\star$, yielding a sequence of upper bounds converging to \(V^\star\). Similarly, UCRL2 and its refinements provide near-minimax regret guarantees in the average-reward and finite-horizon settings; see \cite{JMLR:v11:jaksch10a}, \cite{azar:2017} and \cite{Bourel2020UCRL3}. However, these approaches are tailored to finite state- and action-space MDPs and are not applicable to evaluate the quality of a general policy $\pi$. 
\par
The concept of upper solutions is closely related to martingale duality in optimal control and the information relaxation approach; see \cite{belomestny2018advanced}, \cite{rogers_2007_pathwise} and \cite{BrownSmith2022IR}. This idea has been successfully used in the recent paper \cite{pmlr-v119-shar20a}. This work proposes to use the duality approach to improve the performance of the $Q$-learning algorithm in finite-horizon MDPs through the use of “lookahead” upper and lower bounds. In the subsequent work \cite{el2023weakly}, the authors extend their duality-based ideas to structured weakly coupled MDPs. Furthermore, in \cite{chen2025adversarial}, the authors propose a duality-based algorithm, ADRL, which utilizes neural network approximation for high-dimensional stochastic control problems. There are also recent papers focusing on lookahead-based methods \cite{rosenberg2023planning}, \cite{merlis2024reinforcement}. Specifically, in \cite{rosenberg2023planning}, the authors propose adaptive planning horizons for planning and deep $Q$-learning, choosing the depth of lookahead as a function of state-dependent value estimates. In \cite{merlis2024reinforcement}, the authors propose a regret-optimal algorithm where the agent receives additional stochastic lookahead information (e.g., transition or reward realizations before acting).
\par
The concept of upper solutions also has a connection to distributional RL, as it can be formulated pathwise or using the distributional Bellman operator; see, e.g., \cite{lyle2019comparative}. Development of the distributional counterpart of the upper solution to the Bellman equation is a promising future research area. 

\vspace{-2ex}
\section{Upper Solutions and the Main Concept of \sf{UVIP}}
\label{sec: alg th}
A straightforward approach to bound the policy error $\Delta_{\pi}(x)$ requires the estimation of the optimal value function $V^{\star}(x)$. Recall that $V^{\star}$ is a solution of the Bellman optimality equation \eqref{eq:Bellman_optimality_equation}. 
If the transition kernel $(\MK_a)_{a \in \Aset}$ is known, the standard solution is the value iteration algorithm; see \cite{bertsekas1996stochastic}. In this algorithm, the estimates are constructed recursively via $V_{k+1} = \opM \opT_{\MK}V_{k}$. Due to Banach's fixed point theorem, $\|V_k - V^\star\|_{\Xset} \leq \gamma^{k}\|V_{0} - V^{\star}\|_{\Xset},$ provided that $V_0 \in \sn{B}(\Xset)$. Moreover, $V_{k}(x) \geq V^{\star}(x)$ for any $x \in \Xset$ and $k \in \nset$, provided that $V_0(x) \geq V^{\star}(x)$. For example, if \(\|r^a\|_{\Xset}\leq R_{\max}\) for all \(a\in \Aset\), we can take $V_0(x) = R_{\max}/(1-\gamma)$.
\par
Unfortunately, \eqref{eq:Bellman_optimality_equation} does not allow us to represent $V^\star$ as an expectation and to reduce the problem of estimating $V^\star$ to a stochastic approximation problem. Moreover, if $(\MK^{a})_{a \in \Aset}$ is replaced by its empirical estimate $\hat{\MK}^{a}$, the desired upper-bias property $V_{k}(x) \geq V^{\star}(x)$ is lost. Some recent work (e.g., \cite{efroni2019tight}) suggested a modification of the optimism-based approach applicable in case of unknown $(\MK^{a})_{a \in \Aset}$. Yet this modification contains an additional optimization step, which is unfeasible beyond tabular state- and action-space problems. Therefore, the problem of constructing upper bounds for the optimal value function $V^\star$ and the policy error remains open and highly relevant. In the following, we describe our approach, which is based on the following key assumptions:
\begin{itemize}
\item we consider infinite-horizon MDPs with discount factor $\gamma < 1$;
\item we can sample from the conditional distribution \(\MK^{a}(\cdot | x)\) for any \(x\in \Xset\) and \(a\in \Aset.\)
\end{itemize}
\par
The key concept of our algorithm is \emph{upper solution}, introduced below.
\begin{defi}
We call a function $V^{\rm{up}}$ an \textit{upper solution} to the Bellman optimality equation \eqref{eq:Bellman_optimality_equation} if
\[
V^{\rm{up}}(x)\geq \opM \opT_{\MK}V^{\rm{up}}(x)\,, \forall x \in \Xset\,.
\]
\end{defi}
Upper solutions can be used to build tight upper bounds for the optimal value function \(V^\star.\) Let \(\Phi^{x,a}\in \sn{B}(\Xset)\), $x\in \Xset$, $a\in \Aset$, be a family of \textit{martingale functions} w.r.t. the operator \(\MK^a\), that is, \(\MK^a \Phi^{x,a}(x) = 0\) for all \(a\in \Aset, x \in \Xset\). Define \(V^{\rm{up}}\) as a solution to the following fixed-point equation:
\begin{eqnarray}
\label{eq:v-up}
V^{\rm{up}}(x)= \PE[\max_a \{r^a(x)+ \gamma (V^{\rm{up}}(Y^{x,a})-\Phi(Y^{x,a}))\}], \quad Y^{x,a}\sim \MK^a(\cdot | x).
\end{eqnarray}
In terms of the random Bellman operator \(\widetilde{\opT}_{\MK}\), we can rewrite \eqref{eq:v-up} as \(V^{\rm{up}}=\PE [\opM\widetilde{\opT}_{\MK}(V^{\rm{up}}-\Phi)]\). It is easy to see that \eqref{eq:v-up} defines an upper solution. Indeed, for any $x \in \Xset$,
\begin{eqnarray*}
V^{\rm{up}}(x) &\geq& \max_a\PE[r^a(x)+ \gamma (V^{\rm{up}}(Y^{x,a})-\Phi(Y^{x,a}))] 
\\
&=& \max_a\{r^a(x) + \gamma \MK^a V^{\rm{up}}(x)\} = \opM \opT_{\MK}V^{\rm{up}}(x)\,.
\end{eqnarray*}
Note that unlike the optimal state value function $V^\star$, the upper solution \(V^{\rm{up}}\) is represented as an expectation, which allows us to use various stochastic approximation methods to compute \(V^{\rm{up}}\). Banach's fixed-point theorem implies that for iterates 
\[
V_{k+1}^{\rm{up}}=\PE[\opM\widetilde{\opT}_{\MK}(V_{k}^{\rm{up}}-\Phi)],\quad k \in \nset,
\]
we have convergence \(V_{k}^{\rm{up}}\to V^{\rm{up}}\) as \(k\to \infty.\) Moreover, \(V^{\rm{up}}\) does not depend on \(V_{0}^{\rm{up}}\) and \(V^{\rm{up}}_{k}(x) \geq V^\star(x)\) for any $k \in \nset$, $x \in \Xset$, provided that $V^{\rm{up}}_{0}(x) \geq V^\star(x)$. Given a policy $\pi$ and the corresponding value function $V^\pi$, we set \(\Phi_\pi^{x,a}(y) \eqdef V^\pi(y)-(\MK^a V^\pi)(x)\). It is easy to check that $\MK^a \Phi_\pi^{x,a}(x) = 0$. This leads to the upper value iterative procedure ({\sf UVIP}):
\begin{equation}
\label{eq:pi-up}
\begin{split}
V_{k+1}^{\rm{up}}(x) &= \PE[\opM\widetilde{\opT}_\MK(V_k^{\rm{up}}-\Phi^{x,\cdot}_{\pi})(x)] \\
&= \PE\bigl[\max_a \{r^a(x)+ \gamma (V_k^{\rm{up}}(Y^{x,a})-\Phi_\pi^{x,a}(Y^{x,a}))\}\bigr]\,,
\end{split}
\end{equation}
with \(V_{0}^{\rm{up}}\in \sn{B}(\Xset)\). The algorithm~\ref{alg UVIP} contains the pseudocode of the {\sf UVIP} for MDPs with finite state and action spaces. Several generalizations are discussed in the next section.

\begin{algorithm}[ht]
        \SetAlgoLined
        \KwIn{ $V^\pi$, $
        V_0^{\mathrm{up}}$, $\gamma$, $\varepsilon$}
        \KwResult{$V^{\mathrm{up}}$}
        \For{$x\in \Xset, a \in \Aset$}{ 
                \For{$y \in \Xset$}{
                       $\Phi_\pi^{x,a}(y)=V^{ \pi}(y)-(\MK^a V^{\pi})(x)$;
                }
        }
        $k = 1$; \While{ $\|V_{k}^{\mathrm{up}} - V_{k-1}^{\mathrm{up}}\|_{\Xset} > \varepsilon$}{
            \For{$x\in \Xset$}{ 
                    $V_{k+1}^{\rm{up}}(x)= \PE[\max_a \{r^a(x)+ \gamma (V_k^{\rm{up}}(Y^{x,a})-\Phi^{x,a}_\pi(Y^{x,a}))\}],\quad Y^{x,a}\sim \MK^a(\cdot | x)$;
            }
            $k = k + 1$\;
        }
        $V^{{\rm up}} = V_{k}^{\mathrm{up}}$.
        \caption{\sf{UVIP}}
        \label{alg UVIP}
    \end{algorithm}

Taking \(\Phi^{x,a}(y) \eqdef V^\star(y)-(\MK^a V^\star)(x)\), we get with probability $1$:
\begin{equation}
\label{eq:bellman-as}
\begin{split}
V^\star(x) 
&= (\opM\widetilde{\opT}_{\MK}(V^\star-\Phi^{x,\cdot}))(x) \\
&= \max_a \{r^a(x)+ \gamma (V^\star(Y^{x,a})-\Phi^{x,a}(Y^{x,a}))\}\,,
\end{split}
\end{equation}
that is, \eqref{eq:bellman-as} can be viewed as an almost sure version of the Bellman equation \(V^\star = \opM \opT_{\MK}V^\star.\)


\par
The upper solutions can be used to evaluate the quality of policies and to construct confidence intervals for \(V^\star\). It is clear that 
\[
V^{\pi}(x)\leq V^\star(x)\leq V_{k}^{\rm{up}}(x)
\]
for any \(k \in \nset\) and \(x\in \Xset\), and thus a policy \(\pi\) can be evaluated by computing the difference \(\Delta^{{\rm up}}_{\pi,k}(x)\doteq V_{k}^{\rm{up}}(x)-V^{\pi}(x)\geq \Delta_{\pi}(x)\). Representations \eqref{eq:pi-up} and \eqref{eq:bellman-as} imply
\begin{eqnarray*}
\left\Vert V_{k+1}^{{\rm {up}}}-V^{\star}\right\Vert _{\Xset}\leq\gamma\left\Vert V_{k}^{{\rm {up}}}-V^{\star}\right\Vert _{\Xset}+2\gamma\left\Vert V^{\pi}-V^{\star}\right\Vert _{\Xset},\quad k \in \nset\,.
\end{eqnarray*}
Hence, we derive that $\Delta_\pi^{{\rm up}} \doteq \lim_{k \to \infty} \Delta_{\pi,k}^{{\rm up}}$ satisfies 
\begin{eqnarray}
\label{eq:Deltapi}
\|\Delta_{\pi}\|_{\Xset} \leq \|\Delta_{\pi}^{{\rm up}}\|_{\Xset}\leq \left(1+2\gamma (1-\gamma)^{-1}\right)\|V^\star-V^\pi\|_{\Xset}.
\end{eqnarray}
As a result, \(\Delta_{\pi}^{{\rm up}} = 0\) if \(\pi=\pi^\star\) and the corresponding confidence intervals collapse into a single point.
Moreover, for a policy $\pi$ which is greedy w.r.t. an action-value function $Q^{\pi}(x,a)$, it holds that \(V^\pi(x)\geq V^\star(x)- 2(1-\gamma)^{-1}\|Q^{\pi}-Q^\star\|_{\Xset\times \Aset}\) (see \cite{szepesvari2010algorithms}). Thus, we can rewrite the bound \eqref{eq:Deltapi} in terms of action-value functions  
\[
\|\Delta_{\pi}^{{\rm up}}\|_{\Xset}\leq 2 \left( 1+2\gamma (1-\gamma)^{-1}\right) (1-\gamma)^{-1} \|Q^{\pi}-Q^\star\|_{\Xset\times \Aset}.
\]  
The quantity \(\Delta_{\pi,k}^{{\rm up}}\) can be used to measure the quality of policies \(\pi\) obtained by many well-known algorithms like {\sf Reinforce} (\cite{Williams:92}), {\sf API} (\cite{Bertsekas+Tsitsiklis:1996}), {\sf A2C} (\cite{pmlr-v48-mniha16}) and {\sf DQN} (\cite{dqn}). 

\paragraph{Comparison with PAC-based confidence intervals and IPOC framework.}
We highlight the core differences between our approach and PAC-based confidence intervals. Typically, the latter provide instance-specific bounds (depending on the particular policy $\pi$ and on the properties of the particular algorithm that outputs $\pi$), which additionally depend on problem characteristics that are practically unknown. In contrast, we aim to suggest a \emph{generic} approach to estimating $\Delta^{\pi}(x)$ for an \emph{arbitrary} input policy $\pi$. Our approach can be integrated into the IPOC framework \cite{dann2019policy}, since it provides a suboptimality gap that can be interpreted precisely as an optimality certificate. At the same time, the IPOC approach is not a direct counterpart of the $\sf{UVIP}$ procedure, as IPOC itself does not provide an estimate of the suboptimality gap $\Delta^{\pi}(x)$ (the optimality certificate in the terminology of \cite{dann2019policy}), but instead relies on estimates derived from PAC-style analyses of particular policies $\pi$. Moreover, following the analysis in \cite{dann2019policy}, one can translate the finite-sample bounds on the $\sf{UVIP}$ error given in Theorem~\ref{th: main} into regret bounds. We leave the detailed analysis of the IPOC procedure based on $\sf{UVIP}$ outputs for particular MDP settings as an important direction for future work.

\vspace{-3ex}
\section{Approximate \sf{UVIP}}
\label{sec:fit-vup}
In order to implement the approach described in the previous section, we need to construct empirical estimates for the outer expectation and the one-step transition operator \(\MK^a\) in \eqref{eq:pi-up}. While in the tabular case this boils down to a straightforward Monte Carlo, in the case of infinitely many states we need an additional approximation step.
Algorithm~\ref{alg main} contains the pseudocode of \emph{Approximate \sf{UVIP}} algorithm. Our main assumption is that sampling from $\MK^a(\cdot|x)$ is available for any $a \in \Aset$ and $x \in \Xset$. For simplicity, we assume that the value function $V^{\pi}$ is known, but it can be replaced by its (lower-biased) estimate both in Algorithm~\ref{alg main} and in subsequent theoretical results. We set $G$ as
\begin{align*}
    G_k(x, a, y) = r^a(x) + \gamma\bigl(\widehat{V}_{k}^{\rm{up}}(y) - V^{\pi}(y) + M_1^{-1}\sum_{\ell=1}^{M_1} V^\pi(Y_\ell^{x,a})\bigr).
\end{align*}
The proposed algorithm proceeds as follows. 
\begin{algorithm}[ht]
        \SetAlgoLined
        \KwIn{ Sample $(\smallx_1, \dots, \smallx_N)$; $V^\pi, \widetilde{V}_0^{\mathrm{up}}$, $M_1$, $M_2$, $\gamma$, $\varepsilon$}
        \KwResult{$\widehat{V}^{\mathrm{up}}$}
        Generate $r^a(\smallx_i), Y^{\smallx_i,a}_j \sim \MK^{a}(\cdot | \smallx_i)$ for all $i\in [N]$, $j\in [M_1+M_2]$, \(a\in \Aset\)\;
        $k = 1$;
        \While{ $\sup\limits_{x \in \Xset_N}|\widetilde{V}_{k}^{\mathrm{up}}(x) - \widetilde{V}_{k-1}^{\mathrm{up}}(x)| > \varepsilon$}{
            \For{$a\in \Aset$}{
                \For{$i\in [N]$}{
                    \For{$j\in [M_1+M_2]$}{
                        $\widehat{V}_{k}^{\rm{up}}(Y_j^{\smallx_i,a})=I[\widetilde V_{k}^{\mathrm{up}}](Y_j^{\smallx_i,a})$ with \(I[\cdot](\cdot)\) defined in \eqref{eq:interp}\;
                    }
                    $\overline V^{(i,a)} = M^{-1}_1 \sum\limits_{j=1}^{M_1}V^{\pi}(Y_j^{\smallx_i,a})$\;
                }
            }

            \For{$i \in [N]$}{
                    $
                    \widetilde V_{k+1}^{\mathrm{up}}(\smallx_i)=M_2^{-1}\sum\limits_{j=M_1+1}^{M_1+M_2} \max\limits_{a\in \mathcal{A}}\bigl\{r^a(\smallx_i) + \gamma\bigl(\widehat{V}_{k}^{\rm{up}}(Y_j^{\smallx_i,a}) - V^{\pi}(Y_j^{\smallx_i,a}) + \overline V^{(i,a)}\bigr)\bigr\}
                    $\;
            }

            $k = k + 1$\;
        }
        $\widehat V^{{\rm up}} = \widehat V_{k}^{\mathrm{up}}$.
        \caption{Approximate \sf{UVIP}}
        \label{alg main}
    \end{algorithm}
At the $(k+1)$th iteration, given a previously constructed approximation \(\widehat V^{\rm{up}}_k\), we compute
\begin{eqnarray*}
\widetilde{V}_{k+1}^{\rm{up}}(\smallx_i) = M_2^{-1}\sum_{j=M_1+1}^{M_1+M_2}\max_a \Bigl\{G_k(\smallx_i, a, Y_j^{\smallx_i,a})\Bigr\}\,,
\end{eqnarray*}
where \(\Xset_N = \{\smallx_1,\ldots,\smallx_N\}\) are design points, either deterministic or sampled from some distribution on \(\Xset.\)  
Then the next iterate \(\widehat V^{\rm{up}}_{k+1}\) is obtained via an interpolation scheme based on the points $\widetilde{V}_{k+1}^{\rm{up}}(\smallx_1),\ldots,\widetilde{V}_{k+1}^{\rm{up}}(\smallx_N)$ such that \(\widehat{V}_{k+1}^{\rm{up}}(\smallx_i)=\widetilde{V}_{k+1}^{\rm{up}}(\smallx_i),\) \(i=1,\ldots,N.\) Note that interpolation is needed since $\widehat{V}_{k+1}^{\rm{up}}$ has to be calculated at the (random) points $Y_j^{\smallx_i,a}$, which may not belong to the set $\Xset_{N}$.
In the tabular case when $|\Xset| < \infty$ is not large, one can omit the interpolation and take $\Xsetn = \Xset$. 
\par
\vspace{-0.1cm}
In a more general setting, when $(\Xset, \rho_\Xset)$ is an arbitrary compact metric space, we suggest using an appropriate interpolation procedure. The one described below is particularly useful for our situation, where the function to be interpolated is only Lipschitz continuous (due to the presence of the maximum). The \emph{optimal} central interpolant for a function \(f\in \mathrm{Lip}_{\rho_{\Xset}}(L)\) is defined as
\begin{eqnarray}
\label{eq:interp}
I[f](x) \eqdef (H_f^{\mathrm{low}}(x)+H_f^{\mathrm{up}}(x))/2,
\end{eqnarray}
where
\begin{eqnarray*}
H^{\mathrm{low}}_f(x) =\max_{\ell \in [N]} (f(\smallx_\ell)-L\rho_{\Xset}(x, \smallx_\ell)), \,
H^{\mathrm{up}}_f(x)=\min_{\ell \in [N]} (f(\smallx_\ell)+L \rho_{\Xset}(x,\smallx_\ell)).
\end{eqnarray*}
Note that \(H_f^{\mathrm{low}}(x)\leq f(x)\leq H_f^{\mathrm{up}}(x),\) \(H_f^{\mathrm{low}},H_f^{\mathrm{up}}\in \mathrm{Lip}_{\rho_{\Xset}}(L)\), and hence 
\(I[f]\in  \mathrm{Lip}_{\rho_\Xset}(L).\) An efficient algorithm is proposed in \cite{beliakov2006interpolation} to compute the values of the interpolant \(I[f]\) without knowing $L$ in advance. The so-constructed interpolant achieves the bound 
\begin{eqnarray}
\label{eq:interp-error}
\textstyle{\| f-I[f]\|_{\Xset}\leq L\max_{x\in \Xset}\min_{\ell \in [N]}\rho_{\Xset}(x,\smallx_\ell)}. 
\end{eqnarray}
 The quantity 
 \begin{equation}
\label{eq: covering radius}
\textstyle{\rho(\Xset_N, \Xset) \eqdef  \max_{x\in \Xset}\min_{\ell \in [N]}\rho_{\Xset}(x,\smallx_\ell)}
\end{equation} in the r.h.s. of \eqref{eq:interp-error} is usually called covering radius (also known as the mesh norm or fill radius) of \(\Xset_N\) with respect to \(\Xset\). 

\vspace{-4ex}
\section{Theoretical Results}
\label{sec: theory}
In this section, we analyze the distance between $(\widehat{V}_{k}^{{\rm {up}}})_{k \in \nset}$ and $V^{{\star}}$, where $\widehat{V}_{k}^{\rm{up}}(x)$ is the $k$-th iterate of Algorithm~\ref{alg main}. Recall that \(\Xset_N=\{\smallx_1,\ldots,\smallx_N\}\) is a set of design points (random or deterministic) used in the iterations of Algorithm~\ref{alg main}. First, note that with \(\overline{V}_{k}^{\rm{up}}(x)\eqdef\mathsf{E}\bigl[\widehat{V}_{k}^{\rm{up}}(x)\bigr]\) we have 
\begin{eqnarray}
\overline{V}_{k}^{\rm{up}}(x)\geq \max_a \bigl\{r^a(x)+ \gamma \MK^a\overline{V}_{k-1}^{\rm{up}}(x)\bigr\},\quad x\in \Xset_N, \quad k \in \nset.
\end{eqnarray}
Furthermore, if $\widehat{V}_{0}^{\rm{up}}(x) \geq V^\star(x)$ for $x \in \Xsetn,$ then \( \overline{V}_{k}^{\rm{up}}(x)\geq V^{\star}(x) \) for any $x\in \Xset_N$ and \(k \in \nset\).
Hence, $\widehat{V}_{k}^{\rm{up}}$ is an upper-biased estimate of $V^{\star}$ for any $k\geq 0$.

Before stating our convergence results, we first state a number of technical assumptions. 
\begin{assum}
\label{ass:X}
We suppose that $(\Xset, \rho_\Xset)$ and $(\Aset, \rho_\Aset)$ are compact metric spaces. Moreover, $\Xset \times \Aset$ is equipped with some metric $\rho$ such that $\rho\bigl((x,a),(x',a)\bigr) = \rho_\Xset(x,x')$ for any $x,x' \in \Xset$ and $a \in \Aset$.
\end{assum}

We put special emphasis on the cases when $\Xset$ (resp. $\Aset$) is either finite or $\Xset \subseteq [0,1]^{\dx}$ with $\dx \in \nset$.

\begin{assum}
\label{ass:Pa}
There exists a measurable mapping \(\psi:\) \(\Xset\times \Aset \times \rset^{m} \to \Xset\) such that \(Y^{x,a}=\psi(x,a,\xi),\)
where \(\xi\) is a random variable with values in \(\Xi \subseteq \rset^{m}\) and distribution \(P_\xi\) on \(\Xi,\) that is, \(\psi(x,a,\xi)\sim \MK^a(\cdot|x).\)
\end{assum}

A\ref{ass:Pa} is a reparametrization assumption which is popular in RL, see e.g. \cite{ciosek_expected_PG_JMLR}, \cite{SVG_nips_2015}, \cite{liu2018action-dependent} and the related discussions. This assumption is rather mild, since a large 
class of controlled Markov chains can be represented in the form of random iterative functions, see \cite{douc2018markov}.

\begin{assum}
\label{ass: r-Vpi}
For some positive constant \(R_{\max}\) and all \(a\in \Aset\), $\|r^a\|_{\Xset}\leq R_{\max}\,.$

\end{assum}
\begin{assum}
\label{ass: r-Vpi lip}
For some positive constants $L_{\psi}\leq 1, L_{\max},L_{\pi}$ and all \(a\in \Aset,\) \(\xi \in \Xi\),
\[
\mathrm{Lip}_{\rho_{\Xset}}(r^{a}(\cdot))\leq L_{\max},  \quad \mathrm{Lip}_{\rho}(\psi(\cdot,\cdot,\xi))\leq L_{\psi}, \quad 
\mathrm{Lip}_{\rho}((V^{\pi} \circ \psi)(\cdot,\cdot,\xi)) \leq L_{\pi}\,.
\]
\end{assum}
\begin{rem}
If $|\Xset| < \infty$ and $|\Aset| < \infty$, the assumption A\ref{ass: r-Vpi lip} holds with $\rho_{\Xset}(x,x') = \indiacc{x \neq x'}$, $\rho((x,a), (x',a')) = \indiacc{(x,a) \neq (x',a')}$ and constants $L_{\psi} = 1,\,L_{\max} = R_{\max}$, and $L_{\pi} = R_{\max}/(1-\gamma)$.
\end{rem}
The condition $L_{\psi}\leq 1$ implies a non-explosive behavior of the Markov chain $(X_i)_{i\geq 0}$. This assumption is common in theoretical RL studies, see e.g. \cite{pires2016policy}. If $L_{\psi}< 1$, the corresponding Markov kernel contracts and there exists a unique invariant probability measure, see e.g. \cite{jarner2001locally}.

Suppose that we use an i.i.d. sample \(\boldsymbol{\xi}_k=(\xi_{k,1}\ldots,\xi_{k,M_1+M_2})\sim \mathrm{P}_\xi^{\otimes (M_1+M_2)}\) for each \(k \in [K]\) to generate \(Y^{x,a}_{j}=\psi(x,a,\xi_{k,j}), j \in [M_1+M_2]\) and these samples are independent for different \(k\). For $\varepsilon > 0$, we denote by $\sn{N}(\Xset \times \Aset, \rho, \varepsilon)$ the covering number of the set $\Xset \times \Aset$ w.r.t. metric $\rho$, that is, the smallest cardinality of an $\varepsilon$-net of $\Xset \times \Aset$ w.r.t. $\rho$. Then $\log \sn{N}(\Xset \times \Aset, \rho, \varepsilon)$ is the metric entropy of $\Xset \times \Aset$, and
$$
I_{\mathcal D} \eqdef \int_0^{\mathsf{D}} \sqrt{\log  \sn{N}\bigl(\Xset \times \Aset, \rho, u \bigr)}\, \rmd u
$$
is the Dudley's integral. Here $\mathsf{D} \eqdef \max\limits_{(x,a), (x',a') \in \Xset \times \Aset} \rho((x,a),(x',a'))$. Recall that $\rho(\Xsetn, \Xset)$ defined in \eqref{eq: covering radius} is the covering radius of the set $\Xset_N$ w.r.t. $\Xset$. We now state one of our main theoretical results.
\begin{thm}
\label{th: main}
Let A\ref{ass:X} -- A\ref{ass: r-Vpi lip} hold and suppose that 
$\mathrm{Lip}_{\rho_{\Xset}}(\widehat V^{\rm{up}}_{0}) \leq L_0$
with some constant $L_0 > 0$. Then for any $k \in \nset$ and $\delta \in (0,1)$, it holds with probability at least $1-\delta$ that
\begin{equation}
\label{eq:prob}
\Vert \widehat{V}_{k}^{{\rm {up}}}-V^{{\rm *}}\Vert _{\Xset} \lesssim \gamma^k\bigl\Vert \widehat{V}_{0}^{{\rm {up}}}-V^{{\rm *}}\bigr\Vert _{\Xset} + \left\Vert V^{\pi}-V^{{\rm *}}\right\Vert _{\Xset} +\frac{I_{\mathcal D} +  \mathsf{D} \sqrt{\log(1/\delta)}}{\sqrt{M_1}}  + \rho(\Xset_N, \Xset)\,.
\end{equation}
In the above bound \(\lesssim\) stands for inequality up to a constant depending on $\gamma, L_{\max}, L_{\psi}, L_{\pi}, L_{0}$ and $R_{\max}$. A precise dependence on the aforementioned constants can be found in \eqref{eq:Markov_inequality_explicit_constant} in the Appendix. 
\end{thm}
\vspace{-0.3cm}
\begin{proof}
The proof is given in Section~\ref{sec:proof_main}.
\end{proof}
\vspace{-0.3cm}
Below we specify the result of Theorem~\ref{th: main} for two particular cases of MDPs, which are common in applications. The first one is an MDP with finite state and action spaces, and the second one is an MDP with the state space $\Xset \subseteq [0,1]^\dx$. 
\begin{cor}
\label{cor:cor_discrete_main}
Let $|\Xset|, |\Aset| < \infty$ and assume A\ref{ass:Pa}, A\ref{ass: r-Vpi}. 
Then for any $k \in \nset$ and \(\delta \in (0,1)\) it holds with probability at least $1-\delta$ that
\begin{align*}
\Vert \widehat{V}_{k}^{{\rm {up}}}-V^{{\rm *}}\Vert _{\Xset} &\lesssim \gamma^k\bigl\Vert \widehat{V}_{0}^{{\rm {up}}}-V^{{\rm *}}\bigr\Vert _{\Xset} + \left\Vert V^{\pi}-V^{{\rm *}}\right\Vert _{\Xset}
+ \sqrt{\frac{\log (|\Xset| |\Aset|/\delta)}{M_1}}\,.
\end{align*}
The precise expression for the constants can be found in \eqref{eq:prob_bound_tabular_case} in the Appendix.
\end{cor}
\vspace{-0.3cm}
\begin{proof}
The proof is given in Section~\ref{sec:corollaries_proof_main}.
\end{proof}
\vspace{-0.3cm}
\begin{cor}
\label{cor:cor_continuous_main}
Let $\Xset \subseteq [0,1]^\dx$, $|\Aset|< \infty$, and consider $\rho_{\Xset}(x,x') = \|x - x'\|$, $\rho\bigl((x,a), (x',a')\bigr) = \|x - x'\| + \indiacc{a \neq a'}$. Assume that A\ref{ass:Pa} -- A\ref{ass: r-Vpi lip} hold and let $\Xset_N = \{\smallx_1, \ldots, \smallx_N\}$ be a set of $N$ points independently and uniformly distributed over $\Xset$. If additionally $\mathrm{Lip}_{\rho_{\Xset}}(\widehat V^{\rm{up}}_{0}) \leq L_0$ for some $L_0 > 0$, then for any $k \in \nset$ and \(\delta \in (0,1/2)\) it holds with probability at least $1-\delta$ that
\begin{equation*}
\begin{split}
\Vert \widehat{V}_{k}^{{\rm {up}}}-V^{{\rm *}}\Vert _{\Xset} 
&\lesssim \gamma^k\bigl\Vert \widehat{V}_{0}^{{\rm {up}}}-V^{{\rm *}}\bigr\Vert _{\Xset} + \left\Vert V^{\pi}-V^{{\rm *}}\right\Vert _{\Xset} + \sqrt{\frac{\dx \log (\dx |\Aset|/\delta)}{M_1}}\\
&\quad +\sqrt{\dx}\left(N^{-1} \log(1/\delta) \log N \right)^{1/\dx}\,.
\end{split}
\end{equation*}
The precise expression for the constants can be found in \eqref{eq:prob_bound_subset_r_d} in the Appendix.
\end{cor}
\vspace{-0.3cm}
\begin{proof}
The proof is given in Section~\ref{sec:corollaries_proof_main}.
\end{proof}
\vspace{-0.3cm}
\paragraph{Variance of the estimator and confidence bounds.} Our next step is to bound the variance of the estimator $\widehat{V}_{k}^{{\rm {up}}}(x)$. We additionally assume that $\Xset \times \Aset$ is a parametric class with the metric entropy satisfying the following assumption:
\begin{assum}
\label{ass: covering number}
There exist a constant $C_{\Xset, \Aset} > 1$ such that for any $\varepsilon \in (0,\mathsf{D})$,
$$
\log \sn{N}(\Xset \times \Aset, \rho, \varepsilon) \le C_{\Xset, \Aset} \log (1+1/\varepsilon).
$$
\end{assum}

Denote the r.h.s. of \eqref{eq:prob} by $\sigma_k$, that is,
\begin{equation}
\label{eq:sigma_k_def}
\sigma_k \eqdef \gamma^k\bigl\Vert \widehat{V}_{0}^{{\rm {up}}}-V^{{\rm *}}\bigr\Vert_{\Xset} + \left\Vert V^{\pi}-V^{{\rm *}}\right\Vert _{\Xset} +\frac{I_{\mathcal D} +  \mathsf{D}}{\sqrt{M_1}}  + \rho(\Xset_N, \Xset)\,.
\end{equation}
The next theorem implies that $\mathsf{Var}\bigl[\widehat{V}_{k}^{{\rm {up}}}(x)\bigr]$ can be much smaller than the standard rate \(1/M_2,\) provided that \(V^{\pi}\) is close to \(V^{*}\) and \(M_1,N,K\) are large enough. 
\begin{thm}
\label{th: variance}
Let A\ref{ass:X} -- A\ref{ass: covering number} hold and assume additionally $\mathrm{Lip}_{\rho_{\Xset}}(\widehat V^{\rm{up}}_{0}) \leq L_0$ for some $L_0 > 0$. Then
\begin{equation}
\label{eq:var-bound}
\max_{x\in\Xset} \mathsf{Var}\bigl[\widehat{V}_{k}^{{\rm {up}}}(x)\bigr]\leq  
\ConstC \sigma_k^2 \log(\rme \vee \sigma_k^{-1}) M_2^{-1}\,,
\end{equation}
where the constant $\ConstC$ depends on $C_{\Xset, \Aset}, \gamma, L_{\max}, L_{\psi}, L_{\pi}, L_{0}$ and $R_{\max}$. A precise expression for $\ConstC$ can be found in \eqref{eq:Var_final_bound} in the Appendix.
\end{thm}
\vspace{-0.3cm}
\begin{proof}
The proof is given in Section~\ref{sec:proof_variance_bound}.
\end{proof}
\vspace{-0.3cm}

\begin{cor}
Recall that $\widehat{V}_{k}^{\rm{up}}$ is an upper-biased estimate of $V^{\star}$ in the sense that \( \overline{V}_{k}^{\rm{up}}(x)\geq V^{\star}(x)\) provided $\widehat{V}_{0}^{\rm{up}}(x) \geq V^\star(x)$ for $x \in \Xsetn$. Together with Theorem~\ref{th: variance}, it implies that for any $\delta \in (0,1)$, with probability at least $1-\delta$,
\begin{multline}
\label{eq:conf-bounds}
V^{\pi}(x)\leq V^\star(x)\leq \widehat{V}_{k}^{{\rm {up}}}(x)
\\
+ \sigma_k \sqrt{\ConstC \log(\rme \vee \sigma_k^{-1}) \delta^{-1} M_2^{-1}}+L_V\rho(\Xset_N, \Xset)\indiacc{x \not\in  \Xset_N},\quad x\in\Xset,
\end{multline}
where the constant $L_V$ is given by \eqref{eq:L_V_definition} in the Appendix.
\end{cor}
Note that bounds of the type \eqref{eq:conf-bounds} are known in the literature only in the case of specific policies $\pi$. For example, \cite{wainwright2019} proves bounds of this type for greedy policies in tabular $\mathsf{Q}$-learning. At the same time, \eqref{eq:conf-bounds} holds for an arbitrary policy $\pi$ and a general state space. 
\par
Now we aim to track the dependence of the r.h.s. of \eqref{eq:conf-bounds} on the quantity $\left\Vert V^{\pi}-V^\star \right\Vert_{\Xset}$ for MDPs with finite state and action spaces. The following proposition implies that $\sigma_k$ scales (almost) linearly with $\left\Vert V^{\pi}-V^\star\right\Vert _{\Xset}$.
\begin{prop}
\label{prop:sigma_k_bound}
Let $|\Xset|, |\Aset| < \infty$, assume A\ref{ass:Pa}, A\ref{ass: r-Vpi}, and $\bigl\Vert \widehat{V}_{0}^{{\rm {up}}} \bigr\Vert_{\Xset} \leq R_{\max}(1-\gamma)^{-1}$. Then for $k$ and $M_1$ large enough, it holds that 
\begin{equation}
\label{eq:upper_bound_scaling}
\sigma_k \lesssim \left\Vert V^{\pi}-V^\star \right\Vert _{\Xset} \,.
\end{equation}
The precise bounds for $k$ and $M_1$ can be found in \eqref{eq:k_M_1_bound_appendix}.
\end{prop}
\begin{proof}
The proof is given in Section~\ref{sec:proof_prop_5_1}.
\end{proof}
\vspace{-4ex}
\section{Numerical Results}
\label{sec:num}
In this section, we demonstrate the performance of Algorithm~\ref{alg main} on several tabular and continuous state-space RL problems. Recall that the closer the policy \(\pi\) is to the optimal one \(\pi^\star\), the smaller is the difference between \(V^{\pi}(x)\) and \(V^{\rm{up},\pi}(x)\).
\vspace{-0.3cm}
\paragraph{Discrete state-space MDPs}
We consider $3$ popular tabular environments: Garnet (\cite{10.2307/2584329}), Chain (\cite{rowland2020conditional}) and NRoom (\cite{rlberry}). Detailed descriptions of these environments are provided in Appendix~\ref{sec:envs}. For each environment, we perform $K$ updates of the Value iteration (see Appendix~\ref{sec:envs} for details) with known transition kernel $\MK^a$. We denote the $k$-th step estimate of the action-value function by $\hat{Q}_k(x,a)$ and denote by $\pi_k$ the greedy policy w.r.t. $\hat{Q}_k(x,a)$. Then we evaluate the policies $\pi_k$ with Algorithm~\ref{alg main} for certain iteration numbers $k$. We omit the approximation step because the state space is small. Experimental details are provided in Table~\ref{tab:parameters} in the appendix. Figure~\ref{fig:exp12} displays the gap between $V^{\pi_k}(x)$ and $V^{\rm{up},\pi_k}(x)$, which converges to zero as $\pi^k$ approaches the optimal policy $\pi^\star$. 

In the NRoom environment, we first learn a suboptimal policy $\pi$ using the Value Iteration (VI) algorithm. In the third room, we then replace this policy with a uniformly random policy $\pi_c$ with probability $1/2$. As expected, this modification results in a less efficient policy within that specific room, which, in turn, should increase the upper bounds of our estimation. To demonstrate this effect, we compute precise upper bounds using the UVIP algorithm. As shown in Figure~\ref{fig:exp12}(bottom), UVIP effectively captures the suboptimality of the policy in the third room, while displaying only slight changes in value estimates for the other rooms.

\vspace{-4ex}
\begin{figure}[H]
    \centering
    \subfigure{\includegraphics[width=0.15\textwidth]{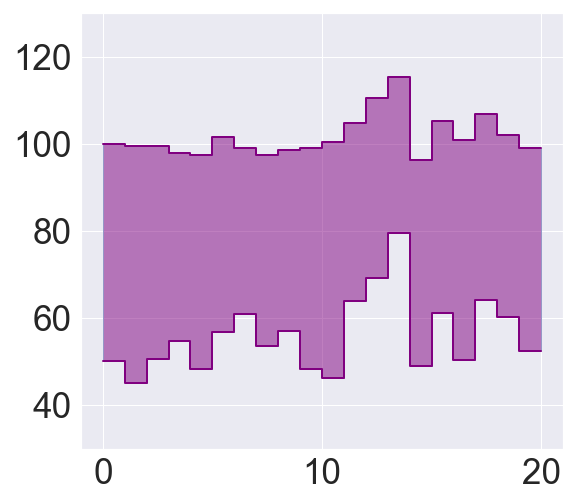}}
    \hspace{-1ex}
    \subfigure{\includegraphics[width=0.15\textwidth]{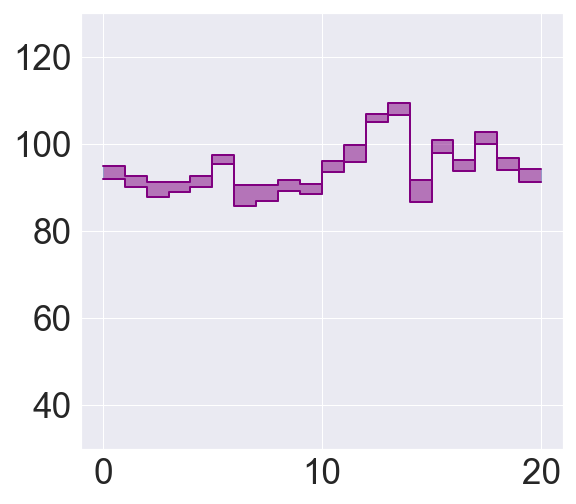}} 
    \hspace{-1ex}
    \subfigure{\includegraphics[width=0.15\textwidth]{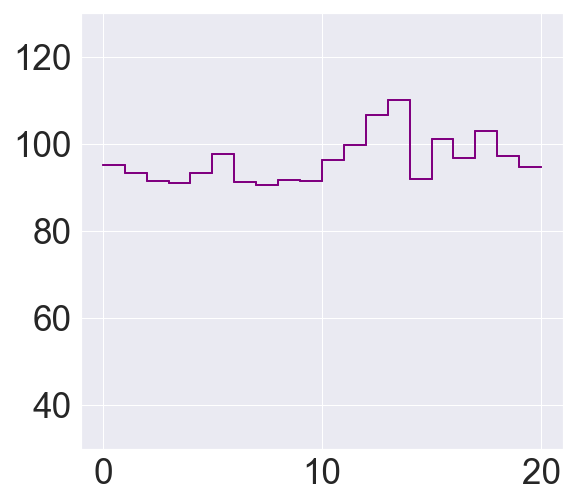}}
    \hspace{2ex}
    \subfigure{\includegraphics[width=0.15\textwidth]{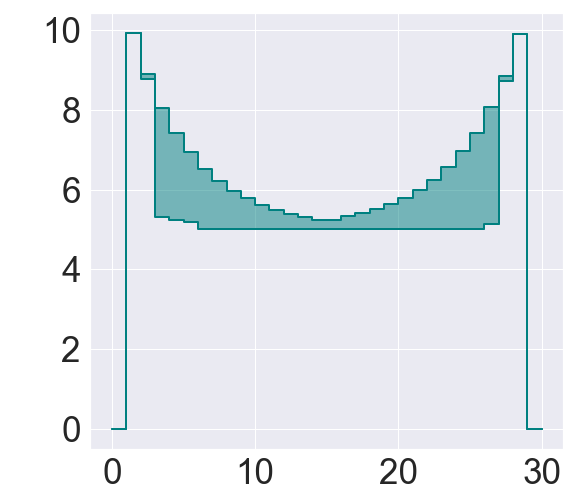}}
    \hspace{-1ex}
    \subfigure{\includegraphics[width=0.15\textwidth]{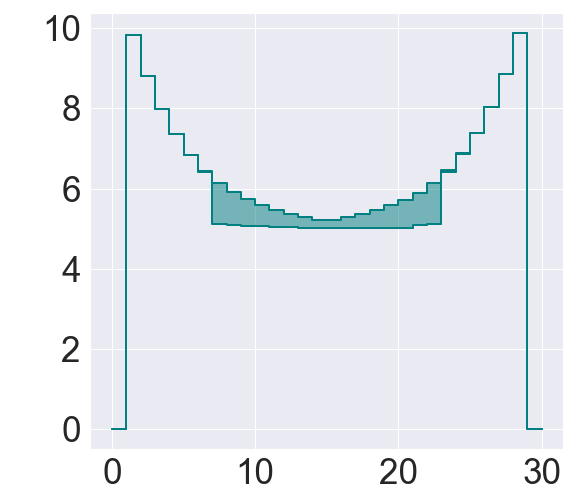}}
    \hspace{-1ex}
    \subfigure{\includegraphics[width=0.15\textwidth]{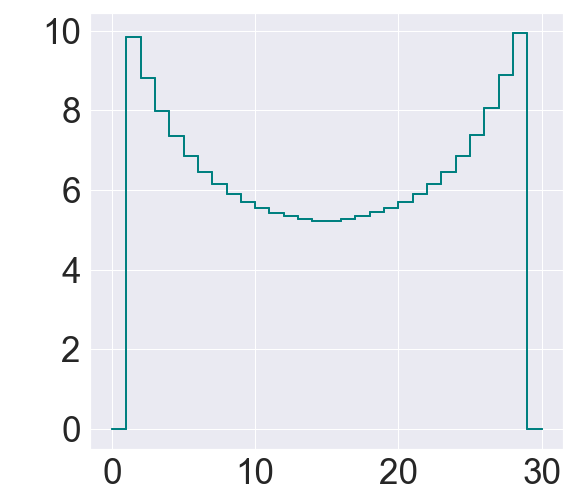}}

    \subfigure{\includegraphics[width=0.45\textwidth]{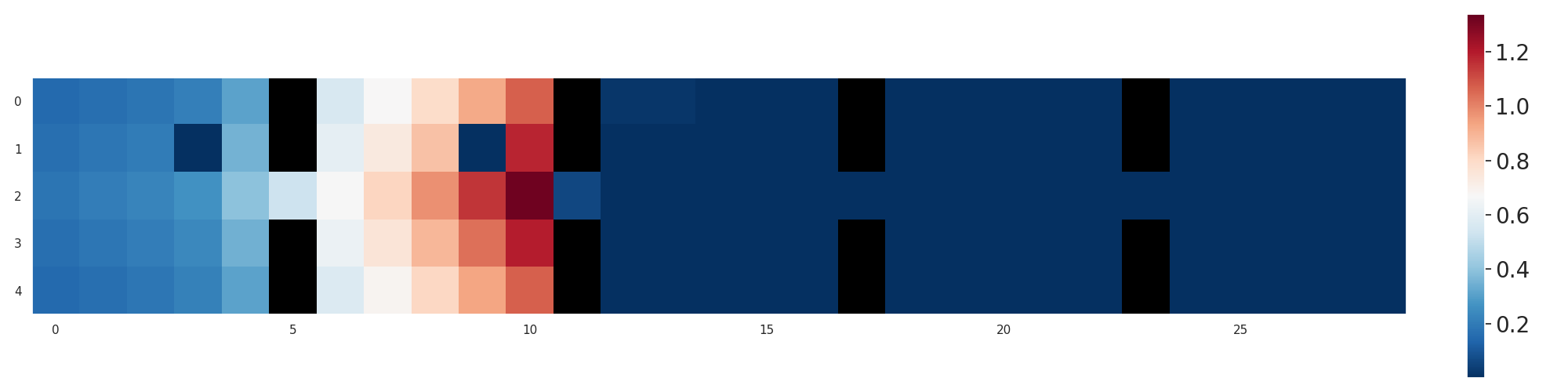}}
    \hspace{2ex}
    \subfigure{\includegraphics[width=0.45\textwidth]{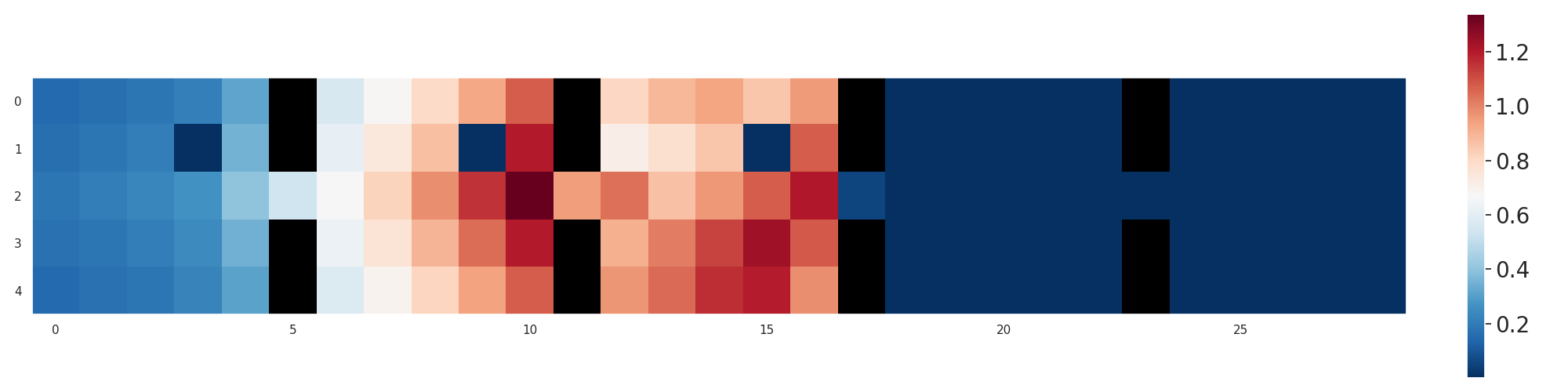}}
    
    \caption{The difference between $V^{\rm{up},\pi_i}(x)$ and $V^{\pi_i}(x)$. The x-axis represents states in a discrete environment for all pictures. Each group of three pictures of the same color illustrates the process of learning the policy from the first iteration to the last.
    \emph{First row}: Evaluation of the policies during the process of Value Iteration for Garnet (left) and Chain environments (right). The policies are the greedy ones corresponding to the function \(Q_i(x, a)\) at the $i$-th step. \emph{Second row}: Comparison of the gap between $V^{\pi}$ and $V^{\rm{up}, \pi}$ for the learned policy $\pi$ and the corrupted policy $\pi_{c}$ in the NRoom environment. The color in these plots represents the value of $V^{\rm{up}, \pi} - V^{\pi}$.}
    \label{fig:exp12}
\end{figure}

\vspace{-4ex}

\paragraph{Continuous state-space MDPs}
In all subsequent experiments, we obtain sample points $(\smallx_1, \dots, \smallx_N)$ in Algorithm~\ref{alg main} from trajectories of the evaluation policy. These points are sufficiently representative (see \cite{kveton2012kernel}, \cite{barreto2016practical}) and explore key areas of the state space. We consider the OpenAI Gym CartPole and Acrobot environments (see \cite{1606.01540}), with their descriptions provided in Appendix~\ref{sec:envs}. For CartPole, we evaluate the A2C algorithm policy $\pi_1$ (\cite{pmlr-v48-mniha16}), the linear deterministic policy (LD) $\pi_2$ described in Appendix~\ref{sec:envs}, and a random uniform policy $\pi_3$. Figure~\ref{fig:exp3} (left) indicates the superior quality of $\pi_2$, a certain instability introduced by A2C training in $\pi_1$, and the low quality of $\pi_3$. We also evaluate a policy for Acrobot given by A2C, as well as a policy from Dueling DQN (\cite{wang2016dueling}) (Fig.~\ref{fig:exp3} (right)). From the plots, we can conclude that both policies are good but far from optimal.
\vspace{-4ex}

\begin{figure}[H]
    \centering
    \subfigure{
        \includegraphics[scale=0.18]{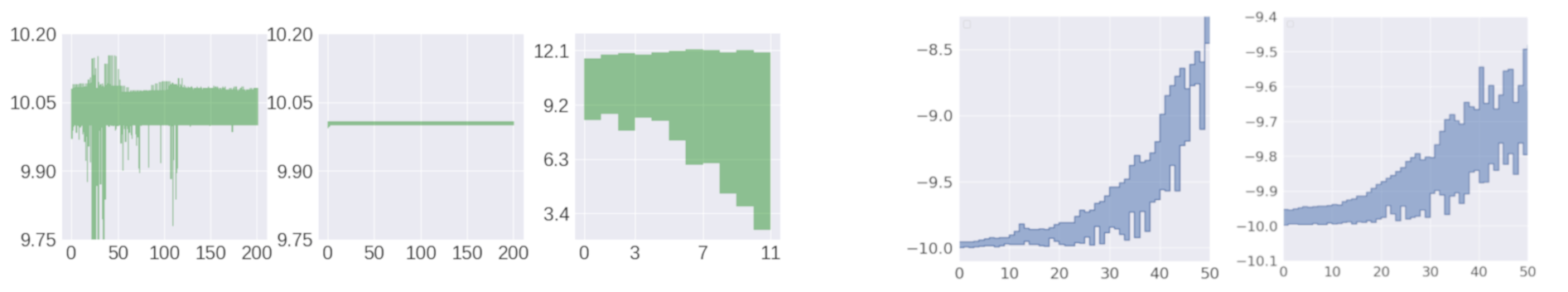}
    }
    \caption{Upper and lower bounds for three different policies. \emph{Left:} For CartPole $\pi_1$, $\pi_2$, $\pi_3$ policies, respectively. For the horizontal axis, we sample a single trajectory according to the policy. \emph{Right:} For Acrobot Dueling DQN and A2C policies, respectively. We evaluate the bounds for the first 50 states of the trajectory for each algorithm.}
    \label{fig:exp3}
\end{figure}
\vspace{-4ex}

Additionally, we compare policies in the TwinRooms environment from the \textit{rlberry} (\cite{rlberry}) library. We obtain two policies $\pi_1$ and $\pi_2$ after running the Kernel-UCBVI (\cite{domingues2022kernelbasedreinforcementlearningfinitetime}) algorithm for $2500$ and $5000$ iteration steps, respectively. The results in Figure~\ref{fig:expTwinRooms} show that after $5000$ learning steps the policy $\pi_2$ has a tighter gap between the lower bound $V^{\pi}$ and the upper bound $V^{\rm{up},\pi}$ on the optimal value function. Also, our upper bounds highlight the regions of the state space that are less studied with our policy.
\vspace{-4ex}
\begin{figure}[H]
    \centering
    \subfigure{\includegraphics[width=0.45\textwidth, height=2.3cm]{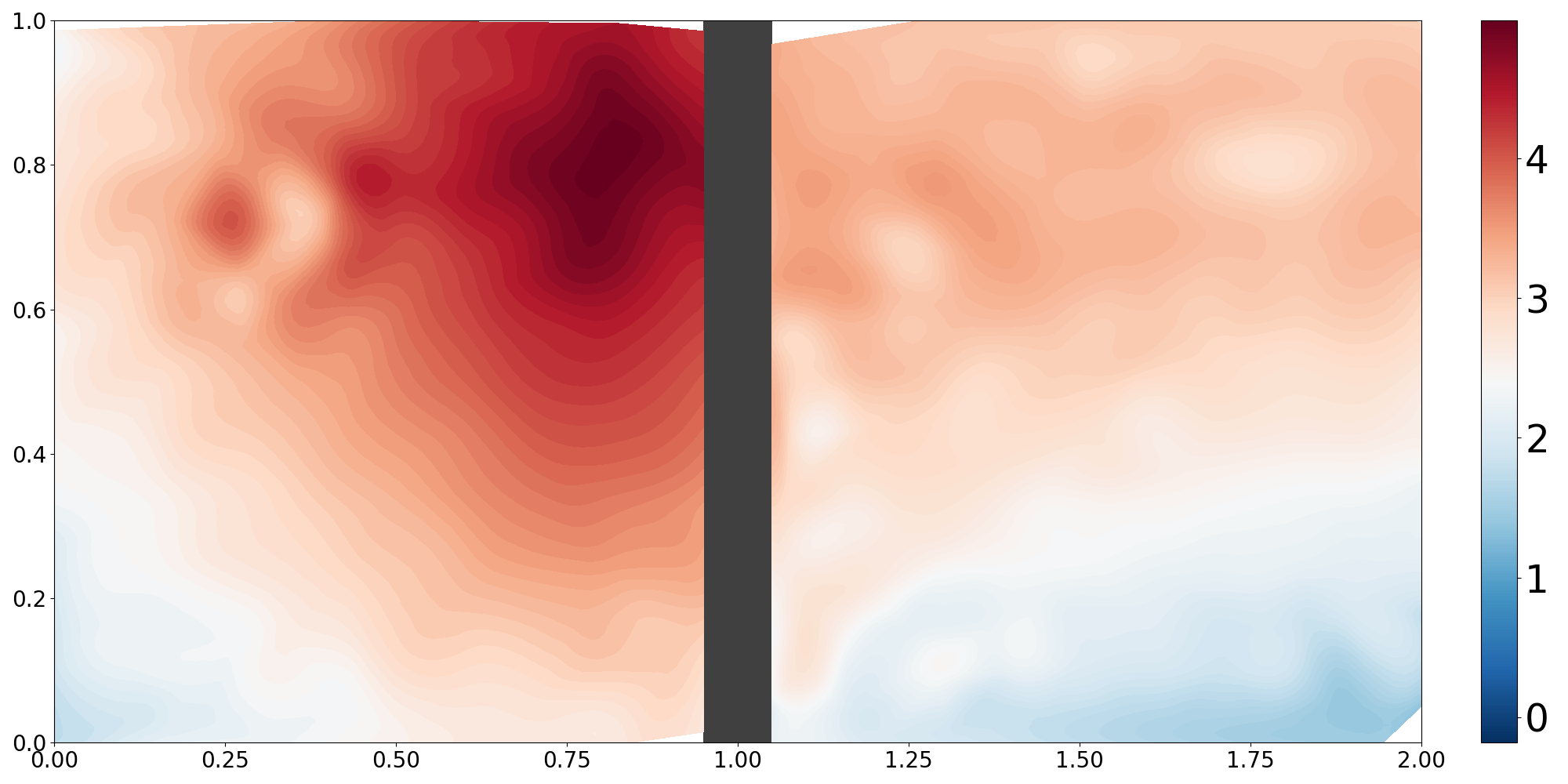}}
    \hspace{2ex}
    \subfigure{\includegraphics[width=0.45\textwidth, height=2.3cm]{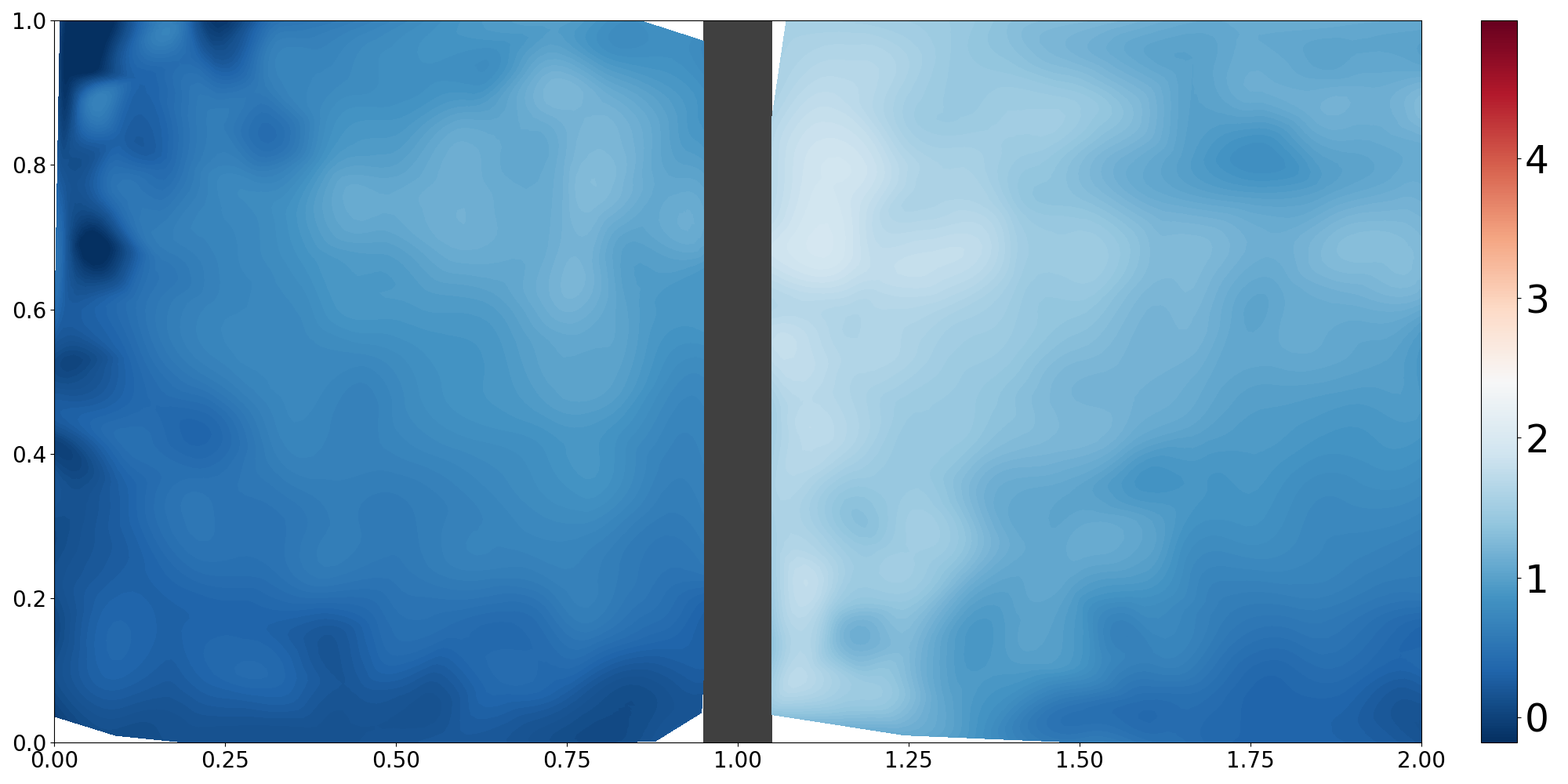}} 
    
    \caption{We illustrate the gap between $V^{\rm{up}, \pi}$ and $V^{\pi}$ in the TwinRooms environment. The color in these plots represents the value of $V^{\rm{up}, \pi} - V^{\pi}$. On the left and right, we show this quantity for $\pi_1$ and $\pi_2$, respectively. We obtain $\pi_1$ and $\pi_2$ after $2500$ and $5000$ learning steps of the Kernel-UCBVI algorithm.}
    \label{fig:expTwinRooms}
\end{figure}
\vspace{-6ex}

\paragraph{Running time of the UVIP algorithm}
To demonstrate that the construction of upper bounds is computationally efficient, we compare the UVIP algorithm with the value function estimation algorithm, Fitted Q-Evaluation (FQE). Specifically, we focus on the running time of these algorithms, ensuring a common convergence criterion is applied to both. Let ${V_k}$ represent the estimates of the lower or upper bound at the $k$-th step of the FQE or UVIP algorithms. Additionally, we select a set of sample points $\Xset_N$ at which convergence will be measured. To this end, we define the quantity
    \begin{align*}
        E_k = \sqrt{\sum\limits_{x \in \Xset_N} \left(1 - {V_{k-1}(x) \over V_k(x)}\right)^2 }\,.
    \end{align*}
We stop iterations of both FQE and UVIP procedures when $E_k \leq 0.01$. After training the policy for $K$ steps, we construct an $\varepsilon$-greedy policy and evaluate it using both algorithms. Results on the TwinRooms environment are summarized in Table~\ref{tab:running_time}. While the UVIP algorithm converges for all instances, it sometimes exhibits significant variance. In contrast, the FQE algorithm fails to converge within the fixed budget for certain seeds. Both algorithms show similar performance in this task.

\vspace{-4ex}
\begin{table}[H]
\caption{Comparison of running time (in seconds) between the UVIP and FQE algorithms on the TwinRooms environment. The evaluation is conducted for different policies after a specified number of training steps of the Kernel-UCBVI algorithm (represented in each row). The reported running times are averaged over 5 seeds. In some instances, the FQE algorithm does not converge within the fixed budget of 400 epochs. Therefore, results are presented separately: one set includes all seeds, while another considers only the converged cases.}
\begin{center}
\begin{tabular}{|c|c|c|c|}
\hline
policy training steps & FQE all seeds & FQE converged & UVIP \\ \hline
$1250$ & $73.60 \pm 13.26$ & $73.60 \pm 13.26$ & $49.23 \pm 66.29$ \\ \hline
$2500$ & $120.85 \pm 94.54$ & $65.56 \pm 14.24$ & $78.36 \pm 77.54$ \\ \hline
$5000$ & $110.06 \pm 85.35$ & $87.06 \pm 65.36$ & $82.26 \pm 107.77$ \\ \hline
\end{tabular}
\end{center}
\label{tab:running_time}
\end{table}

\vspace{-6ex}
\section{Conclusion and Future Work}
\label{sec:conclusion}
In this work, we propose a new approach towards model-free evaluation of the agent's policies in RL, based on upper solutions to the Bellman optimality equation \eqref{eq:Bellman_optimality_equation}. To the best of our knowledge, {\sf UVIP} is the first procedure that allows us to construct non-asymptotic confidence intervals for the optimal value function $V^\star$ based on the value function corresponding to an arbitrary policy $\pi$. In our analysis, we consider only infinite-horizon MDPs and assume that sampling from the conditional distribution \(\MK^{a}(\cdot | x)\) is feasible for any \(x\in \Xset\) and \(a\in \Aset.\) A promising future research direction is to generalize {\sf UVIP} to the case of finite-horizon MDPs by combining it with the idea of real-time dynamic programming (see \cite{efroni2019tight}).
Moreover, plain Monte Carlo estimates are not necessarily the most efficient way to estimate the outer expectation in Algorithm~$1$. Other stochastic approximation techniques could also be applied to approximate the solution of \eqref{eq:v-up}.

\vspace{-2ex}
\section*{Acknowledgments}
This work was supported by the grant for research centers in the field of AI provided by the Ministry of Economic Development of the Russian Federation in accordance with the agreement 000000C313925P4E0002 and the agreement with HSE University № 139-15-2025-009. This research was supported in part through computational resources of HPC facilities at HSE University \cite{kostenetskiy2021hpc}.

\bibliography{biblio-rl}

@article{liu2018action-dependent,
  title={Action-{D}ependent {C}ontrol {V}ariates for {P}olicy {O}ptimization via {S}tein's {I}dentity},
  author={Liu, Hao and Feng, Yihao and Mao, Yi and Zhou, Dengyong and Peng, Jian and Liu, Qiang},
  journal={arXiv preprint arXiv:1710.11198},
  year={2017}
}

@book{bertsekas1996stochastic,
  title={Stochastic {O}ptimal {C}ontrol: the {D}iscrete-{T}ime {C}ase},
  author={Bertsekas, Dimitri and Shreve, Steven E},
  volume={5},
  year={1996},
  publisher={Athena Scientific}
}

@inproceedings{kostenetskiy2021hpc,
  title={{HPC} {R}esources of the {H}igher {S}chool of {E}conomics},
  author={Kostenetskiy, PS and Chulkevich, RA and Kozyrev, VI},
  booktitle={Journal of Physics: Conference Series},
  volume={1740},
  number={1},
  pages={012050},
  year={2021},
  organization={IOP Publishing}
}

@article{wainwright2019,
  title={Variance-{R}educed {Q}-learning is {M}inimax {O}ptimal},
  author={Wainwright, Martin J},
  journal={arXiv preprint arXiv:1906.04697},
  year={2019}
}

@inproceedings{SVG_nips_2015,
 author = {Heess, Nicolas and Wayne, Gregory and Silver, David and Lillicrap, Timothy and Erez, Tom and Tassa, Yuval},
 booktitle = {Advances in Neural Information Processing Systems},
 editor = {C. Cortes and N. Lawrence and D. Lee and M. Sugiyama and R. Garnett},
 pages = {},
 publisher = {Curran Associates, Inc.},
 title = {Learning {C}ontinuous {C}ontrol {P}olicies by {S}tochastic {V}alue {G}radients},
 volume = {28},
 year = {2015}
}

@article{ciosek_expected_PG_JMLR,
  author  = {Kamil Ciosek and Shimon Whiteson},
  title   = {Expected {P}olicy {G}radients for {R}einforcement {L}earning},
  journal = {Journal of Machine Learning Research},
  year    = {2020},
  volume  = {21},
  number  = {52},
  pages   = {1-51}
}

@article{rogers_2007_pathwise,
    author = {Rogers, L. C. G.},
    title = {Pathwise {S}tochastic {O}ptimal {C}ontrol},
    journal = {SIAM Journal on Control and Optimization},
    volume = {46},
    number = {3},
    pages = {1116-1132},
    year = {2007}
}

@article{sutton_1988_td,
author = {Sutton, Richard},
year = {1988},
month = {08},
pages = {9-44},
title = {Learning to {P}redict by the {M}ethod of {T}emporal {D}ifferences},
volume = {3},
journal = {Machine Learning}
}

@inproceedings{azar:2017,
  title={Minimax {R}egret {B}ounds for {R}einforcement {L}earning},
  author =       {Mohammad Gheshlaghi Azar and Ian Osband and R{\'e}mi Munos},
  booktitle = 	 {Proceedings of the 34th International Conference on Machine Learning},
  pages = 	 {263--272},
  year = 	 {2017},
  editor = 	 {Precup, Doina and Teh, Yee Whye},
  volume = 	 {70},
  month = 	 {06--11 Aug},
  publisher =    {PMLR}
}

@article{BrownSmith2022IR,
  author  = {Brown, David B. and Smith, James E.},
  title   = {Information {R}elaxations and {D}uality in {S}tochastic {D}ynamic {P}rograms: {A} {R}eview and {T}utorial},
  journal = {Foundations and Trends in Optimization},
  volume  = {5},
  number  = {3},
  pages   = {246--339},
  year    = {2022}
}

@article{chen2025adversarial,
  title={Adversarial {R}einforcement {L}earning: {A} {D}uality-{B}ased {A}pproach {T}o {S}olving {O}ptimal {C}ontrol {P}roblems},
  author={Chen, Nan and Liu, Mengzhou and Wang, Xiaoyan and Zhang, Nanyi},
  journal={arXiv preprint arXiv:2506.00801},
  year={2025}
}

@inproceedings{rosenberg2023planning,
  title={Planning and {L}earning with {A}daptive {L}ookahead},
  author={Rosenberg, Aviv and Hallak, Assaf and Mannor, Shie and Chechik, Gal and Dalal, Gal},
  booktitle={Proceedings of the AAAI Conference on Artificial Intelligence},
  volume={37},
  editor={Yiling Chen and Jennifer Neville},
  number={8},
  pages={9606--9613},
  year={2023}
}

@book{Bertsekas+Tsitsiklis:1996,
author = {Bertsekas, Dimitri P. and Tsitsiklis, John N.},
title = {Neuro-{D}ynamic {P}rogramming},
year = {1996},
isbn = {1886529108},
publisher = {Athena Scientific},
edition = {1st},
}

@inproceedings{Bourel2020UCRL3,
  title={Tightening {E}xploration in {U}pper {C}onfidence {R}einforcement {L}earning},
  author =       {Bourel, Hippolyte and Maillard, Odalric and Talebi, Mohammad Sadegh},
  booktitle = 	 {Proceedings of the 37th International Conference on Machine Learning},
  pages = 	 {1056--1066},
  year = 	 {2020},
  editor = 	 {III, Hal Daumé and Singh, Aarti},
  volume = 	 {119},
  month = 	 {13--18 Jul},
  publisher =    {PMLR}
}

@inproceedings{q_learning_efficient,
 author = {Jin, Chi and Allen-Zhu, Zeyuan and Bubeck, Sebastien and Jordan, Michael I},
 booktitle = {Advances in Neural Information Processing Systems},
 editor = {S. Bengio and H. Wallach and H. Larochelle and K. Grauman and N. Cesa-Bianchi and R. Garnett},
 pages = {},
 publisher = {Curran Associates, Inc.},
 title = {Is {Q}-{L}earning {P}rovably {E}fficient?},
 volume = {31},
 year = {2018}
}

@article{JMLR:v11:jaksch10a,
  author  = {Thomas Jaksch and Ronald Ortner and Peter Auer},
  title   = {Near-{O}ptimal {R}egret {B}ounds for {R}einforcement {L}earning},
  editor = {S. Kakade},
  journal = {Journal of Machine Learning Research},
  year    = {2010},
  volume  = {11},
  number  = {51},
  pages   = {1563-1600}
}

@inproceedings{merlis2024reinforcement,
 author = {Merlis, Nadav},
 booktitle = {Advances in Neural Information Processing Systems},
 editor = {A. Globerson and L. Mackey and D. Belgrave and A. Fan and U. Paquet and J. Tomczak and C. Zhang},
 pages = {64523--64581},
 title = {Reinforcement {L}earning with {L}ookahead {I}nformation},
 volume = {37},
 year = {2024}
}

@inproceedings{el2023weakly,
 author = {El Shar, Ibrahim and Jiang, Daniel},
 booktitle = {Advances in Neural Information Processing Systems},
 editor = {A. Oh and T. Naumann and A. Globerson and K. Saenko and M. Hardt and S. Levine},
 pages = {43931--43950},
 title = {Weakly {C}oupled {D}eep {Q}-{N}etworks},
 volume = {36},
 year = {2023}
}

@inproceedings{dann2019policy,
  title={Policy {C}ertificates: {T}owards {A}ccountable {R}einforcement {L}earning},
  author =       {Dann, Christoph and Li, Lihong and Wei, Wei and Brunskill, Emma},
  booktitle = 	 {Proceedings of the 36th International Conference on Machine Learning},
  pages = 	 {1507--1516},
  year = 	 {2019},
  editor = 	 {Chaudhuri, Kamalika and Salakhutdinov, Ruslan},
  volume = 	 {97},
  month = 	 {09--15 Jun},
  publisher =    {PMLR}
}

@book{sutton:book:2018,
	author = {Sutton, R. S. and Barto, Andrew G.},
	edition = {Second},
	publisher = {The MIT Press},
	title = {Reinforcement {L}earning: {A}n {I}ntroduction},
	year = {2018}}

@inproceedings{tsitsiklis:td:1997,
 author = {Tsitsiklis, John and Van Roy, Benjamin},
 booktitle = {Advances in Neural Information Processing Systems},
 editor = {M.C. Mozer and M. Jordan and T. Petsche},
 pages = {},
 publisher = {MIT Press},
 title = {Analysis of {T}emporal-{D}iffference {L}earning with {F}unction {A}pproximation},
 volume = {9},
 year = {1996}
}

@book{puterman2014markov,
  title={Markov {D}ecision {P}rocesses: {D}iscrete {S}tochastic {D}ynamic {P}rogramming},
  author={Puterman, Martin L},
  year={2014},
  publisher={John Wiley \& Sons}
}

@article{dqn,
  author = {Mnih, Volodymyr and Kavukcuoglu, Koray and Silver, David and Rusu, Andrei A. and Veness, Joel and Bellemare, Marc G. and Graves, Alex and Riedmiller, Martin A. and Fidjeland, Andreas and Ostrovski, Georg and Petersen, Stig and Beattie, Charles and Sadik, Amir and Antonoglou, Ioannis and King, Helen and Kumaran, Dharshan and Wierstra, Daan and Legg, Shane and Hassabis, Demis},
  journal = {Nature},
  number = 7540,
  pages = {529-533},
  title = {Human-{L}evel {C}ontrol {T}hrough {D}eep {R}einforcement {L}earning.},
  volume = 518,
  year = 2015
}

@InProceedings{pmlr-v48-mniha16,
  title = 	 {Asynchronous {M}ethods for {D}eep {R}einforcement {L}earning},
  author = 	 {Mnih, Volodymyr and Badia, Adria Puigdomenech and Mirza, Mehdi and Graves, Alex and Lillicrap, Timothy and Harley, Tim and Silver, David and Kavukcuoglu, Koray},
  booktitle = 	 {Proceedings of The 33rd International Conference on Machine Learning},
  pages = 	 {1928--1937},
  year = 	 {2016},
  editor = 	 {Balcan, Maria Florina and Weinberger, Kilian Q.},
  volume = 	 {48},
  address = 	 {New York, New York, USA},
  month = 	 {20--22 Jun},
  publisher =    {PMLR},
}

@article{Williams:92,
  author = {Williams, R. J.},
  journal = {Machine Learning},
  pages = {229--256},
  title = {Simple {S}tatistical {G}radient-{F}ollowing {A}lgorithms for {C}onnectionist {R}einforcement {L}earning},
  volume = 8,
  year = 1992
}

@inproceedings{wang2016dueling,
  title={Dueling {N}etwork {A}rchitectures for {D}eep {R}einforcement {L}earning},
  author={Wang, Ziyu and Schaul, Tom and Hessel, Matteo and Hasselt, Hado and Lanctot, Marc and Freitas, Nando},
  editor={Balcan, Maria Florina and Weinberger, Kilian Q.},
  booktitle={International Conference on Machine Learning},
  pages={1995--2003},
  year={2016},
  volume={48},
  organization={PMLR}
}

@book{douc2018markov,
  title={Markov {C}hains},
  author={Douc, Randal and Moulines, Eric and Priouret, Pierre and Soulier, Philippe},
  year={2018},
  publisher={Springer}
}

@article{10.2307/2584329,
	author = {Archibald, T. W. and McKinnon, K. I. M. and Thomas, L. C.},
	date = {1995/03/01},
	id = {Archibald1995},
	isbn = {1476-9360},
	journal = {Journal of the Operational Research Society},
	number = {3},
	pages = {354--361},
	title = {On the {G}eneration of {M}arkov {D}ecision {P}rocesses},
	volume = {46},
	year = {1995},
}

@article{1606.01540,
  title={Open{AI} {G}ym},
  author={Brockman, Greg and Cheung, Vicki and Pettersson, Ludwig and Schneider, Jonas and Schulman, John and Tang, Jie and Zaremba, Wojciech},
  journal={arXiv preprint arXiv:1606.01540},
  year={2016}
}

@article{beliakov2006interpolation,
title = {Interpolation of {L}ipschitz {F}unctions},
journal = {Journal of Computational and Applied Mathematics},
volume = {196},
number = {1},
pages = {20-44},
year = {2006},
issn = {0377-0427},
author = {Gleb Beliakov},
}

@InProceedings{pires2016policy,
  title={Policy {E}rror {B}ounds for {M}odel-{B}ased {R}einforcement {L}earning with {F}actored {L}inear {M}odels},
  author = 	 {Ávila Pires, Bernardo and Szepesvári, Csaba},
  booktitle = 	 {29th Annual Conference on Learning Theory},
  pages = 	 {121--151},
  year = 	 {2016},
  editor = 	 {Feldman, Vitaly and Rakhlin, Alexander and Shamir, Ohad},
  volume = 	 {49},
  address = 	 {Columbia University, New York, New York, USA},
  month = 	 {23--26 Jun},
  publisher =    {PMLR},
}

@article{szepesvari2010algorithms,
  title={Algorithms for {R}einforcement {L}earning},
  author={Szepesv{\'a}ri, Csaba},
  journal={Synthesis Lectures on Artificial Intelligence and Machine Learning},
  volume={4},
  number={1},
  pages={1--103},
  year={2010},
  publisher={Morgan \& Claypool Publishers}
}

@book{belomestny2018advanced,
  title={Advanced Simulation-Based Methods for Optimal Stopping and Control: With Applications in Finance},
  author={Belomestny, Denis and Schoenmakers, John},
  year={2018},
  publisher={Springer}
}

@article{Reznikov,
    author = {Reznikov, A. and Saff, E. B.},
    title = {The {C}overing {R}adius of {R}andomly {D}istributed {P}oints on a {M}anifold},
    journal = {International Mathematics Research Notices},
    volume = {2016},
    number = {19},
    pages = {6065-6094},
    year = {2015},
    month = {12}
}

@book {Vershynin,
    AUTHOR = {Vershynin, Roman},
     TITLE = {High-{D}imensional {P}robability},
    VOLUME = {47},
 PUBLISHER = {Cambridge University Press, Cambridge},
      YEAR = {2018},
     PAGES = {xiv+284},
      ISBN = {978-1-108-41519-4},
   MRCLASS = {60-01 (60B05 60B20 60E15 60Fxx 62H25)},
  MRNUMBER = {3837109},
MRREVIEWER = {Sasha Sodin}
}

@InProceedings{pmlr-v119-shar20a,
  title = 	 {Lookahead-{B}ounded {Q}-learning},
  author =       {Shar, Ibrahim El and Jiang, Daniel},
  booktitle = 	 {Proceedings of the 37th International Conference on Machine Learning},
  pages = 	 {8665--8675},
  year = 	 {2020},
  editor = 	 {III, Hal Daumé and Singh, Aarti},
  volume = 	 {119},
  month = 	 {13--18 Jul},
  publisher =    {PMLR},
}

@inproceedings{efroni2019tight,
	author = {Efroni, Yonathan and Merlis, Nadav and Ghavamzadeh, Mohammad and Mannor, Shie},
	booktitle = {Advances in Neural Information Processing Systems},
	editor = {H. Wallach and H. Larochelle and A. Beygelzimer and F. d\textquotesingle Alch\'{e}-Buc and E. Fox and R. Garnett},
	publisher = {Curran Associates, Inc.},
	title = {Tight {R}egret {B}ounds for {M}odel-{B}ased {R}einforcement {L}earning with {G}reedy {P}olicies},
	volume = {32},
	year = {2019},
}

@inproceedings{lyle2019comparative,
  title={A {C}omparative {A}nalysis of {E}xpected and {D}istributional {R}einforcement {L}earning},
  author={Lyle, Clare and Bellemare, Marc G and Castro, Pablo Samuel},
  booktitle={Proceedings of the AAAI Conference on Artificial Intelligence},
  volume={33},
  editor={Pascal Van Hentenryck and Zhi-Hua Zhou},
  number={01},
  pages={4504--4511},
  year={2019}
}

@InProceedings{rowland2020conditional,
  title = 	 {Conditional {I}mportance {S}ampling for {O}ff-{P}olicy {L}earning},
  author =       {Rowland, Mark and Harutyunyan, Anna and van Hasselt, Hado and Borsa, Diana and Schaul, Tom and Munos, Remi and Dabney, Will},
  booktitle = 	 {Proceedings of the Twenty Third International Conference on Artificial Intelligence and Statistics},
  pages = 	 {45--55},
  year = 	 {2020},
  editor = 	 {Chiappa, Silvia and Calandra, Roberto},
  volume = 	 {108},
  month = 	 {26--28 Aug},
  publisher =    {PMLR},
}

@article{jarner2001locally,
    title={Locally {C}ontracting {I}terated {F}unctions and {S}tability of {M}arkov {C}hains},
    volume={38},
    number={2}, 
    journal={Journal of Applied Probability}, 
    author={Jarner, S. F. and Tweedie, R. L.},
    year={2001}, 
    pages={494–507}
}

@inproceedings{antos2007fitted,
 author = {Antos, Andr\'{a}s and Szepesv\'{a}ri, Csaba and Munos, R\'{e}mi},
 booktitle = {Advances in Neural Information Processing Systems},
 editor = {J. Platt and D. Koller and Y. Singer and S. Roweis},
 pages = {},
 publisher = {Curran Associates, Inc.},
 title = {Fitted {Q}-{I}teration in {C}ontinuous {A}ction-{S}pace {MDP}s},
 volume = {20},
 year = {2007}
}

@misc{rlberry,
    author = {Domingues, Omar Darwiche and Flet-Berliac, Yannis and Leurent, Edouard and M{\'e}nard, Pierre and Shang, Xuedong and Valko, Michal},
    month = {10},
    title = {{rlberry - {A} {R}einforcement {L}earning {L}ibrary for {R}esearch and {E}ducation}},
    url = {https://github.com/rlberry-py/rlberry},
    year = {2021}
}

@inproceedings{domingues2022kernelbasedreinforcementlearningfinitetime,
  title={Kernel-{B}ased {R}einforcement {L}earning: {A} {F}inite-{T}ime {A}nalysis},
  author =       {Domingues, Omar Darwiche and Menard, Pierre and Pirotta, Matteo and Kaufmann, Emilie and Valko, Michal},
  booktitle = 	 {Proceedings of the 38th International Conference on Machine Learning},
  pages = 	 {2783--2792},
  year = 	 {2021},
  editor = 	 {Meila, Marina and Zhang, Tong},
  volume = 	 {139},
  month = 	 {18--24 Jul},
  publisher =    {PMLR}
}

@inproceedings{kveton2012kernel,
  title={Kernel-{B}ased {R}einforcement {L}earning on {R}epresentative {S}tates},
  author={Kveton, Branislav and Theocharous, Georgios},
  booktitle={Proceedings of the AAAI Conference on Artificial Intelligence},
  volume={26},
  number={1},
  editor={Jörg Hoffmann and Bart Selman},
  pages={977--983},
  year={2012}
}

@article{barreto2016practical,
  author  = {Andr{{\'e}} M.S. Barreto and Doina Precup and Joelle Pineau},
  title   = {Practical {K}ernel-{B}ased {R}einforcement {L}earning},
  journal = {Journal of Machine Learning Research},
  year    = {2016},
  volume  = {17},
  number  = {67},
  pages   = {1--70},
}
\bibliographystyle{spmpsci}

\appendix
\section{Proof of the Main Results}
\label{sec:proof_of_main_res}

Throughout this section we will use additional notation. Let $\psi_2(x) = e^{x^2} - 1$, $x \in \rset$. For r.v. $\eta$ we denote $\| \eta \|_{\psi_2} \eqdef \inf\{t > 0: \EE{\exp\{\eta^2/t^2\}} \le 2\}$ the Orlicz 2-norm. We say that $\eta$ is a \emph{sub-Gaussian random variable} if $\| \eta \|_{\psi_2}  < \infty$. In particular, this implies that for some constants $C, c > 0$, $\PP(|\eta| \geq t) \le 2 \exp\{-c t^2/\|\eta \|_{\psi_2}^2\}$ and $\PE^{1/p}[|\eta|^p] \le C \sqrt{p} \|\eta \|_{\psi_2}$ for all $p \geq 1$. Consider a random process $(X_t)_{t \in T}$ on a metric space $(T, \mathsf{d})$. We say that the process has \emph{sub-Gaussian increments} if there exists $K \geq 0$ such that
$$
\| X_t - X_s \|_{\psi_2} \le K \mathsf{d}(t,s), \quad \forall t, s \in T.
$$

We start from the following proposition.
\begin{prop}
\label{prop: emp proc}
Under A\ref{ass:X} -- A\ref{ass: r-Vpi lip} for any $M \in \nset$ 
and $p \geq 1$
\begin{eqnarray*}
\label{eq: max of emp proc lp}
\PE^{1/p} \Bigl[\Bigl\Vert \frac{1}{M}\sum_{l=1}^{M}[V^{\pi}(\psi(\cdot,\cdot,\xi_{l}))-\mathsf{E}V^{\pi}(\psi(\cdot,\cdot,\xi_{l}))]\Bigr \Vert_{\Xset\times\mathcal{A}}^p \Bigr] \lesssim \frac{L_{\pi} I_{\mathcal D} +  \{ L_{\pi} \mathsf{D}  +  R_{\max}/(1-\gamma) \} \sqrt{p}}{\sqrt{M}}. 
\end{eqnarray*}
\end{prop}
\begin{proof}
We apply empirical process methods. To simplify notation, we denote
$$
Z(x, a) = \frac{1}{\sqrt{M}} \sum_{\ell=1}^{M}[V^{\pi}(\psi(x,a,\xi_{\ell}))-\mathsf{E}V^{\pi}(\psi(x,a,\xi_{\ell}))], \, (x,a) \in \Xset \times \Aset,
$$
that is, $Z(x,a)$ is a random process on the metric space $(\Xset \times \Aset, \rho)$. Below we show that the process $Z(x,a)$ has sub-Gaussian increments. To show this, let us introduce for $\ell \in [M]$
$$
Z_{\ell} \eqdef [V^{\pi}(\psi(x,a,\xi_{\ell}))-\mathsf{E}V^{\pi}(\psi(x,a,\xi_{\ell}))] - [V^{\pi}(\psi(x',a',\xi_{\ell}))-\mathsf{E}V^{\pi}(\psi(x',a',\xi_{\ell}))]\,.
$$
Clearly, by A\ref{ass: r-Vpi lip},
$$
\|Z_{\ell}\|_{\psi_2} \lesssim L_{\pi}\rho((x,a),(x',a'))\,,
$$
that is, $Z_\ell$ is a sub-Gaussian r.v.\ for any $\ell \in [M]$. Since $Z(x,a) - Z(x',a') = M^{-1/2} \sum_{\ell=1}^M Z_\ell$ is a sum of independent sub-Gaussian r.v., we may apply \cite[Proposition 2.6.1 and Eq. (2.16)]{Vershynin} to obtain that $Z(x,a)$ has sub-Gaussian increments with parameter $K \asymp L_{\pi}$. Fix some $(x_0, a_0) \in \Xset \times \Aset$. By the triangle inequality, 
\begin{equation}
\label{triangular inequality}
 \sup_{(x,a) \in \Xset \times \Aset} |Z(x,a)| \le \sup_{(x,a), (x',a') \in \Xset \times \Aset} |Z(x,a) - Z(x',a')| + Z(x_0, a_0).   
\end{equation}
By Dudley's integral inequality, e.g. \cite[Theorem 8.1.6]{Vershynin}, for any $\delta \in (0,1)$,
\begin{equation*}
  \sup_{(x,a), (x',a') \in \Xset \times \Aset} |Z(x,a) - Z(x',a')| \lesssim L_{\pi} \bigl[ I_{\mathcal D} +  \mathsf{D} \sqrt{\log(2/\delta)}\bigr ]\,
\end{equation*}
holds with probability at least $1 - \delta$. 
Again, under A\ref{ass: r-Vpi}, $Z(x_0,a_0)$ is a sum of i.i.d. bounded centered random variables with $\psi_2$-norm bounded by $R_{\max}/(1- \gamma)$. Hence, applying Hoeffding's inequality, e.g. \cite[Theorem 2.6.2.]{Vershynin}, for any $\delta \in (0,1)$, 
$$
|Z(x_0,a_0)| \lesssim R_{\max} \sqrt{\log(1/\delta)}/(1- \gamma)
$$
holds with probability $1-\delta$. The last two inequalities and \eqref{triangular inequality} imply the statement.
\end{proof} 

\vspace{-4ex}
\subsection{Proof of Theorem \ref{th: main}}
\label{sec:proof_main}
Fix $p \geq 2$ and denote for any $k \in \nset$, $\Mk_k \eqdef \PE^{1/p}[\Vert \widehat{V}_{k}^{{\rm {up}}}-V^{{\rm *}}\Vert_{\Xset}^p]$. For any $x \in \Xset$, we introduce 
\begin{equation*}
\widetilde{V}_{k+1}^{{\rm {up},\pi}}(x) = \frac{1}{M_{2}}\sum_{j=M_{1}+1}^{M_{1}+M_{2}}\max_{a}\left\{ r^{a}(x)+\gamma\Bigl(\widehat{V}_{k}^{{\rm {up}}}(Y_{j}^{x,a})-V^{\pi}(Y_{j}^{x,a})+\frac{1}{M_{1}}\sum_{\ell=1}^{M_{1}}V^{\pi}(Y_{\ell}^{x,a})\Bigr)\right\} .
\end{equation*}
Recall that $Y_j^{x,a} = \psi(x,a,\xi_{k,j}), j \in [M_1+M_2]$ for independent random variables $(\xi_{k,j})$, thus we can write
\begin{equation}
\label{eq: Vtilde}
\widetilde{V}_{k+1}^{{\rm {up},\pi}}(x) = \frac{1}{M_{2}}\sum_{j=M_{1}+1}^{M_{1}+M_{2}} R_{k}^{x}(\xi_{k,j};\xi_{k,1},\ldots,\xi_{k,M_{1}})\,.
\end{equation}
We first calculate $L_{k+1} \eqdef \mathrm{Lip}_{\rho}(\widetilde{V}_{k+1}^{{\rm {up},\pi}})$ for any $k \in \nset$. Since under A\ref{ass: r-Vpi lip},
$\mathrm{Lip}_{\rho}((V^\pi\circ\psi)(\cdot,\cdot,\xi)) \leq  L_{\pi}$, and using \eqref{eq: Vtilde},
\begin{equation}
\label{eq:Lipsh_constant_recurrence}
L_{k+1} \leq L_{\max} + \gamma(L_{k}L_{\psi} + 2L_{\pi})\,.   
\end{equation}
Expanding \eqref{eq:Lipsh_constant_recurrence} and using the assumptions of Theorem~\ref{th: main}, we obtain
$$
L_{k+1} \leq \frac{L_{\max} + 2\gamma L_{\pi}}{1-\gamma L_{\psi}} + (\gamma L_{\psi})^{k}L_{0}\,, \quad k \in \nset\,.
$$
Using that $\gamma L_{\psi} < 1$, the maximal Lipschitz constant of $\widetilde{V}_{k}^{{\rm {up},\pi}}(x), k \in \nset$ is uniformly bounded by
\begin{equation}
\label{eq:L_V_definition}
L_V = \frac{L_{\max} + 2\gamma L_{\pi}}{1-\gamma L_{\psi}} + L_{0}\,.
\end{equation}
Using \eqref{eq: Vtilde} and ~\eqref{eq:bellman-as}, for any \(x \in \Xset\) and $j = M_1 + 1, \ldots, M_1+M_2$. 
\begin{equation*}
\begin{split}
&\PE^{1/p}[|R_{k}^{x}(\xi_{k,j};\xi_{k,1},\ldots,\xi_{k,M_{1}}) - V^{*}(x)|^p] \leq \\
&\PE^{1/p}\left[\left|\max_{a}\left\{ r^{a}(x)+\gamma\Bigl(\widehat{V}_{k}^{{\rm {up}}}(Y_{j}^{x,a})-V^{\pi}(Y_{j}^{x,a})+M_1^{-1}\sum\nolimits_{\ell=1}^{M_{1}}V^{\pi}(Y_{\ell}^{x,a})\Bigr)\right\}\right.\right. - \\
&\qquad\left.\left.\max_a \left\{r^a(x)+ \gamma (V^\star(Y^{x,a})-V^\star(Y^{x,a}) + \MK^{a}V^\star(x)\right\}\right|^p\right] .
\end{split}
\end{equation*}
Hence, with Minkowski's inequality and $|\MK^{a}V^\star(x) - \PE V^{\pi}(\psi(x,a,\cdot))| \leq \|V^{\pi}-V^{*}\|_{\Xset}$, we get
\begin{equation}
\label{eq: R - Vstar}
\begin{split}
&\PE^{1/p}[|R_{k}^{x}(\xi_{k,j};\xi_{k,1},\ldots,\xi_{k,M_{1}}) - V^{*}(x)|^p] \leq \gamma \Mk_k +2\gamma\left\Vert V^{\pi}-V^{{\rm *}}\right\Vert _{\Xset}
\\
&\qquad\qquad\qquad+\gamma \PE^{1/p}\Bigl[\Bigl\Vert M_1^{-1}\sum\nolimits_{\ell=1}^{M_{1}}[V^{\pi}(\psi(\cdot,\cdot,\xi_{k,\ell})) - 
\mathsf{E}V^{\pi}(\psi(\cdot,\cdot,\xi_{k,\ell}))]\Bigr\Vert _{\Xset\times\mathcal{A}}^p \Bigr] .
\end{split}
\end{equation}
To analyze the last term we use empirical process methods. By Proposition \ref{prop: emp proc}, we get
\begin{equation*}
 \PE^{1/p}\Bigl [\Bigl\Vert \frac{1}{M_{1}}\sum_{\ell=1}^{M_{1}}[V^{\pi}(\psi(\cdot,\cdot,\xi_{k,\ell}))-\mathsf{E}V^{\pi}(\psi(\cdot,\cdot,\xi_{k,\ell}))]\Bigr\Vert _{\Xset\times\mathcal{A}}^p \Bigr] \lesssim  \frac{L_{\pi} I_{\mathcal D} +  \{ L_{\pi} \mathsf{D}  +  R_{\max}/(1-\gamma)\} \sqrt{p}}{\sqrt{M_1}}.
\end{equation*}
Furthermore, with \eqref{eq:interp-error} we construct a Lipschitz interpolant $\widehat{V}_{k+1}^{{\rm {up}}}$ such that
\begin{eqnarray*}
|\widehat{V}_{k+1}^{{\rm {up}}}(x)-\widetilde{V}_{k+1}^{{\rm {up},\pi}}(x)|&\lesssim&   L_{k+1} \rho(\Xset_N, \Xset)\,.
\end{eqnarray*}
Combining the above estimates, we get 
\begin{eqnarray*}
\Mk_{k+1} &\lesssim & \gamma \Mk_k +\gamma\left\Vert V^{\pi}-V^{{\rm *}}\right\Vert _{\Xset}
+\gamma \frac{L_{\pi} I_{\mathcal D} +  \{ L_{\pi} \mathsf{D}  +  R_{\max}/(1-\gamma) \} \sqrt{p}}{\sqrt{M_1}} +L_{k+1}\rho(\Xset_N, \Xset)\,.
\end{eqnarray*}
Iterating this inequality,
\begin{multline}
\label{eq:lp}
\PE^{1/p}[\Vert \widehat{V}_{k}^{{\rm {up}}}-V^{{\rm *}}\Vert _{\Xset}^p] \lesssim \gamma^k\bigl\Vert \widehat{V}_{0}^{{\rm {up}}}-V^{{\rm *}}\bigr\Vert _{\Xset} + \frac{\gamma}{1-\gamma} \left\Vert V^{\pi}-V^{{\rm *}}\right\Vert _{\Xset} +
\\
\frac{\gamma L_{\pi} I_{\mathcal D} +  \gamma\{ L_{\pi} \mathsf{D}  +  R_{\max}/(1-\gamma) \} \sqrt{p} }{\sqrt{M_1}(1 - \gamma)}  + \frac{L_V}{1-\gamma}\rho(\Xset_N, \Xset)\,.
\end{multline}
Applying Markov's inequality with $p \asymp \log{(1/\delta)}$, we get that for any $k \in \nset$ and $\delta \in (0,1)$,
\begin{equation}
\label{eq:Markov_inequality_explicit_constant}
\begin{split}
&\Vert \widehat{V}_{k}^{{\rm {up}}}-V^{{\rm *}}\Vert_{\Xset} 
\lesssim \gamma^k\bigl\Vert \widehat{V}_{0}^{{\rm {up}}}-V^{{\rm *}}\bigr\Vert _{\Xset} + \frac{\gamma}{1-\gamma} \left\Vert V^{\pi}-V^{{\rm *}}\right\Vert _{\Xset} + \\
&\quad \frac{\gamma L_{\pi} I_{\mathcal D} +  \gamma\{ L_{\pi} \mathsf{D}  +  R_{\max}/(1-\gamma)\}\sqrt{\log(1/\delta)} }{\sqrt{M_1}(1 - \gamma)}  + \frac{L_V}{1-\gamma}\rho(\Xset_N, \Xset)\,.
\end{split}
\end{equation}
holds with probability at least $1 - \delta$, where the constant $L_V$ is given in \eqref{eq:L_V_definition}. This yields the statement of the theorem.

\vspace{-4ex}
\subsection{Proof of Corollary~\ref{cor:cor_discrete_main} and Corollary~\ref{cor:cor_continuous_main}}
\label{sec:corollaries_proof_main}
\begin{proof}[Proof of Corollary~\ref{cor:cor_discrete_main}]
Consider $\rho((x,a), (x',a')) = \indiacc{(x,a) \neq (x',a')}$ and $\Xset_N = \Xset$, that is, we bypass the approximation step. Then $\mathsf{D} = 1$, $I_{\mathcal D} \lesssim \sqrt{\log(|\Xset| |\Aset|)}$, $\rho(\Xset_N, \Xset) = 0$, and $r^{a}(\cdot)$ is Lipschitz w.r.t. $\rho_{\Xset}$ with $L_{\max} \leq R_{\max}$. Moreover, one can take $L_{\psi} = 1$ and $L_{\pi} = R_{\max}/(1-\gamma)$ in Assumption A\ref{ass: r-Vpi lip}. Hence,  A\ref{ass:X} -- A\ref{ass: r-Vpi lip} are valid and one may apply Theorem~\ref{th: main}. Bound \eqref{eq:Markov_inequality_explicit_constant} in this case writes as
\begin{multline}
\label{eq:prob_bound_tabular_case}
\Vert\widehat{V}_{k}^{{\rm {up}}}-V^{{\rm *}}\Vert_{\Xset} 
\lesssim \gamma^k\bigl\Vert \widehat{V}_{0}^{{\rm {up}}}-V^{{\rm *}}\bigr\Vert _{\Xset} + \frac{\gamma}{1-\gamma} \left\Vert V^{\pi}-V^{{\rm *}}\right\Vert _{\Xset} +\\ 
\frac{\gamma R_{\max}(\sqrt{\log(|\Xset| |\Aset|)} + 2)\sqrt{\log{(1/\delta)}}}{\sqrt{M_1}(1-\gamma)^2}\,.  
\end{multline}
\end{proof}
\begin{proof}[Proof of Corollary~\ref{cor:cor_continuous_main}]
It is easy to see that $\mathsf{D} \leq \sqrt{\dx} + 1$, $I_{\mathcal D} \lesssim \sqrt{\dx\log |\Aset|} + \sqrt{\dx \log \dx}$.
Proposition \ref{prop: covering rad} implies that for any $\delta \in (0,1)$,
$\rho(\Xsetn, \Xset) \lesssim \sqrt{\dx} \left( N^{-1} \log(1/\delta) \log N  \right)^{1/\dx}$. Substituting into \eqref{eq:Markov_inequality_explicit_constant}, we obtain
\begin{multline}
\label{eq:prob_bound_subset_r_d}
\Vert\widehat{V}_{k}^{{\rm {up}}}-V^{{\rm *}}\Vert_{\Xset} 
\lesssim \gamma^k\bigl\Vert \widehat{V}_{0}^{{\rm {up}}}-V^{{\rm *}}\bigr\Vert _{\Xset} + \frac{\gamma}{1-\gamma} \left\Vert V^{\pi}-V^{{\rm *}}\right\Vert_{\Xset} + \frac{L_V\sqrt{\dx} \left( N^{-1} \log(1/\delta) \log N  \right)^{1/\dx}}{1-\gamma} \\ 
+\frac{\gamma L_{\pi}(\sqrt{\dx\log |\Aset|} + \sqrt{\dx \log \dx}) +  \gamma\{ L_{\pi} \sqrt{\dx} +  R_{\max}/(1-\gamma)\}\sqrt{\log(1/\delta)} }{\sqrt{M_1}(1 - \gamma)}\,,
\end{multline}
where $L_V$ is given in \eqref{eq:L_V_definition}.
\end{proof}

\vspace{-4ex}
\subsection{Proof of Theorem \ref{th: variance}}
\label{sec:proof_variance_bound}
We use the definition of $\widetilde{V}_{k+1}^{{\rm {up},\pi}}(x)$ and $R_{k}^{x}(\xi_{k,j};\xi_{k,1},\ldots,\xi_{k,M_{1}})$ from Theorem~\ref{th: main}. To simplify notation, we denote $\xiv_{k,M_1} = (\xi_{k,1}, \ldots, \xi_{k, M_1})$ and $\xiv_{k,M_2} = (\xi_{k,M_1+1}, \ldots, \xi_{k, M_1+M_2})$. In this notation $\xiv_k = (\xiv_{k,M_1}, \xiv_{k,M_2})$. Recall that, by construction, $\widetilde{V}_{k+1}^{{\rm {up},\pi}}(x)$ can be evaluated only at the points $x \in \{\smallx_1,\dots,\smallx_N\}$. By definition,
\begin{equation}
\label{eq: V up representation}
\widehat{V}_{k+1}^{{\rm {up}}}(x) = \min_{\ell \in [N]} \bigl\{ \widetilde{V}_{k+1}^{{\rm {up},\pi}}(\smallx_\ell) + L_{V} \rho_{\Xset}(\smallx_\ell, x) \bigr\}\,,
\end{equation}
where the constant $L_V$ is given in \eqref{eq:L_V_definition}. We rewrite $\widetilde{V}_{k+1}^{{\rm {up},\pi}}(x)$ as follows
\begin{eqnarray*}
\widetilde{V}_{k+1}^{{\rm {up},\pi}}(x)&=&\frac{1}{M_{2}}\sum_{j=M_{1}+1}^{M_{1}+M_{2}} \{ R_{k}^{x}(\xi_{k,j}; \xiv_{k, M_1})  - \PE[R_{k}^{x}(\xi_{k,j};\xiv_{k,M_1})] \}  \\
&+&  \PE[R_{k}^{x}(\xi;\xiv_{k,M_1})]
=: T_k^x(\xiv_k) +  \PE[R_{k}^{x}(\xi;\xiv_{k,M_1})],
\end{eqnarray*}
where $\xi$ is an i.i.d. copy of $\xi_{k,j}$. Conditioned on $\overline{\sn{G}}_k \eqdef \sn{G}_{k-1} \cup \sigma\{\xiv_{k, M_1}\}$, $T_k^x(\xiv_{k, M_1},  \xiv_{k,M_2})$ is the sum of i.i.d. centered random variables.
In what follows, we will often omit the arguments $\xiv_{k, M_1}$ and/or $\xiv_{k, M_2}$ from the notation of functions $T_k^x$. Using representation \eqref{eq: V up representation},
\begin{align*}
\Var[\widehat{V}_{k+1}^{{\rm {up}}}(x)] &= \Var\left[\min_{\ell \in [N]} \bigl\{ \widetilde{V}_{k+1}^{{\rm {up},\pi}}(\smallx_\ell) + L_{V} \rho_{\Xset}(\smallx_\ell, x) \bigr\}\right] \\
&\leq \PE\left[\left(\min_{\ell \in [N]} \bigl\{ \widetilde{V}_{k+1}^{{\rm {up},\pi}}(\smallx_\ell) + L_{V} \rho_{\Xset}(\smallx_\ell, x) \bigr\} - \min_{\ell \in [N]}\bigl\{\PE[R_{k}^{x}(\xi;\xiv_{k,M_1})] + L_{V} \rho_{\Xset}(\smallx_\ell, x)\bigr\}\right)^2\right]\,. 
\end{align*}
Hence, from the previous inequality and the definition of $\widetilde{V}_{k+1}^{{\rm {up},\pi}}(x)$,
\begin{eqnarray*}
\Var[\widehat{V}_{k+1}^{{\rm {up}}}(x)] \leq \PE[\sup_{\ell \in [N]} \vert T_k^{\smallx_\ell}(\xiv_{k, M_1},  \xiv_{k,M_2}) \vert^2] \leq \PE[\sup_{x \in \Xset} \vert  T_k^x(\xiv_{k, M_1},  \xiv_{k,M_2})\vert^2].
\end{eqnarray*}
To estimate the right-hand side of the previous inequality we again apply the empirical process method. We first note that for any $x, x' \in \Xset$,
\begin{eqnarray}
\label{eq: T lip}
\sup_{\xiv \in \Xi^{M_1+M_2} }|T_k^x(\xiv) - T_k^{x'}(\xiv)| \leq L_{T} \rho_\Xset(x,x'),
\end{eqnarray}
where 
\begin{equation}
\label{eq:L_T_const}
L_{T} = L_{\max} + \gamma(L_{V} + 2L_{\pi})\,.
\end{equation}
Now we freeze the coordinates  $\xiv_{M_1}$ and consider $T_k^x(\xiv_{M_1}, \cdot)$ as a function on $\Xi^{M_2}$, parametrized by $x \in \Xset$. Introduce a parametric class of functions
$$
\mathcal{T}_{k, \xiv_{M_1}} \eqdef \bigl \{T_k^x(\xiv_{M_1},  \cdot): \Xi^{M_2} \to \rset, x \in \Xset \bigr\}\,.
$$
For notational simplicity we will omit dependencies on $k$ and $\xiv_{M_1}$, and simply write $T^x(\cdot) = T_k^x(\xiv_{M_1},\cdot)$. Note that the functions in $\mathcal{T}_{k, \xiv_{M_1}}$ are Lipschitz w.r.t. the uniform metric 
\[
\rho_{\mathcal{T}_{k,\xiv_{M_1}}}(T^x(\cdot), T^{x'}(\cdot)) = \sup_{\xiv_{M_2} \in \Xi^{M_2}} |T^x(\xiv_{M_2}) - T^{x'}(\xiv_{M_2})|, \quad T^x(\cdot),\,T^{x'}(\cdot) \in \mathcal{T}_{k,\xiv_{M_1}}\,.
\]
To estimate $\diam(\mathcal{T}_{k,\xiv_{M_1}})$ we proceed as follows. Denote $\tilde{R}_k^x(\xi;\xiv_{k,M_1}) = R_{k}^{x}(\xi;\xiv_{k,M_1})  - \PE\left[R_{k}^{x}(\xi;\xiv_{k,M_1})\right]$. Using \eqref{eq: R - Vstar}, we get an upper bound
\begin{equation}
\label{eq:r_k_bound}
\begin{split}
&\left|\tilde{R}_k^x(\xi;\xiv_{k,M_1})\right| 
\lesssim \gamma \Vert \widehat{V}_{k}^{{\rm {up}}}-V^{{\rm *}}\Vert_{\Xset}+2\gamma\left\Vert V^{\pi}-V^{{\rm *}}\right\Vert _{\Xset} + \\ 
&\qquad\qquad \gamma\Bigl\Vert M_1^{-1}\sum_{l=1}^{M_{1}}[V^{\pi}(\psi(\cdot,\cdot,\xi_{k,l}))-\mathsf{E}V^{\pi}(\psi(\cdot,\cdot,\xi_{k,l}))]\Bigr\Vert _{\Xset\times\mathcal{A}}.
\end{split}
\end{equation}
We denote the right-hand side of this inequality by $R_k^\star$. Clearly, $R_k^\star$ is an $\overline{\sn{G}}_k$-measurable function (recall that $\overline{\sn{G}}_k = \sn{G}_{k-1} \cup \sigma\{\xiv_{k, M_1}\}$). We may conclude that $\diam(\mathcal{T}_{k,\xiv_{M_1}}) \le 2R_k^\star$. Furthermore, by \eqref{eq: T lip}, its covering number can be bounded as
$$
\sn{N}(\mathcal{T}_{k,\xiv_{M_1}}, \rho_\mathcal T, \varepsilon) \le \sn{N}(\Xset \times \Aset, \rho, \varepsilon/L_T ).
$$
It is also easy to check that $\bigl(T^x(\xiv_{k, M_2})\bigr), T^x \in \mathcal{T}_{k,\xiv_{M_1}}$ is a sub-Gaussian process on $(\mathcal{T}_{k,\xiv_{M_1}}, \rho_T)$ with
$$
\Vert T^x - T^{x'} \Vert_{\psi_2} \lesssim \rho_\mathcal T(T^x, T^{x'}). 
$$
Applying the tower property, we get
$$
\PE[ \sup_{x \in \Xset} \vert T_k^x(\xiv_{k, M_2}) \vert^2] \le \PE[\PE[\sup_{x \in \Xset} \vert  T_k^x(\xiv_{k, M_2}) \vert^2 \vert \overline{\sn{G}}_k]].
$$
Using Dudley's integral inequality, e.g. \cite[Theorem 8.1.6]{Vershynin} and assumption A\ref{ass: covering number},
\begin{eqnarray*}
& \PE[ \sup_{x \in \Xset} \vert T_k^x(\xiv_{k, M_1},\xiv_{k, M_2}) \vert^2 \vert \overline{\sn{G}}_k] = \PE[ \sup_{T^x \in \mathcal{T}_{k,\xiv_{M_1}}} \vert T^x(\xiv_{k, M_2}) \vert^2 \vert \overline{\sn{G}}_k] \\
& \qquad \qquad \qquad   \lesssim \PE[\vert T^{x_0}( \xiv_{k, M_2}) \vert^2 \vert \overline{\sn{G}}_k] + \frac{1}{M_2} \Bigl \{L_{T}\sqrt{C_{\Xset, \Aset}} \int_0^{2 R_k^\star/L_{T}} \sqrt{\log(1 + 1/\varepsilon)} \rmd \varepsilon  + R_k^\star \Bigr\}^2,
\end{eqnarray*}
where $x_0 \in \Xset$ is some fixed point. To estimate the first term in the right-hand side of the previous inequality we apply Hoeffding's inequality. We obtain
$$
\PE[\vert T^{x_0}(\xiv_{k, M_2}) \vert^2 \vert \overline{\sn{G}}_k] \lesssim \frac{(R_k^\star)^2}{M_2}.
$$
Applying Proposition \ref{prop analysis}, we get
$$
\int_0^{2 R_k^\star/L_T} \sqrt{\log(1 + 1/\varepsilon)} \rmd \varepsilon  \lesssim (R_k^\star \sqrt{\log(1 + 1/R_k^\star)} + R_k^\star)/L_T .
$$
The last two inequalities imply
$$
\PE[ \sup_{x \in \Xset} \vert T_k^x(\xiv_{k, M_2}) \vert^2 \vert \overline{\sn{G}}_k] \lesssim C_{\Xset, \Aset}\frac{(R_k^{\star})^2 + (R_k^{\star})^2 \log(1 + 1/R_k^{\star}) }{M_2}.  
$$
Since for $x > 0$ and $\varepsilon \in (0,1]$
$$
\log(1 + x) \le  \varepsilon^{-1} x^\varepsilon, 
$$
we obtain
$$
\PE[ \sup_{x \in \Xset} \vert T_k^x \vert^2] \lesssim C_{\Xset, \Aset}\frac{\PE[(R_k^{\star})^2] + \PE[(R_k^{\star})^{2 - \varepsilon}]/\varepsilon}{M_2}\,.
$$
Using \eqref{eq:r_k_bound}, we get for any $p \geq 1$
\begin{equation*}
\begin{split}
\PE^{1/p}\bigl[(R_k^{\star})^p\bigr] 
&\leq \gamma \PE^{1/p}\bigl[\Vert \widehat{V}_{k}^{{\rm {up}}}-V^{{\rm *}}\Vert^{p}_{\Xset}\bigr] + 2\gamma\left\Vert V^{\pi}-V^{{\rm *}}\right\Vert _{\Xset} + \\
&\qquad \gamma\PE^{1/p}\left[\Bigl\Vert M_1^{-1}\sum_{l=1}^{M_{1}}[V^{\pi}(\psi(\cdot,\cdot,\xi_{k,l}))-\mathsf{E}V^{\pi}(\psi(\cdot,\cdot,\xi_{k,l}))]\Bigr\Vert^{p}_{\Xset\times\mathcal{A}}\right].
\end{split}
\end{equation*}
Thus, applying \eqref{eq:lp} and Proposition~\ref{prop: emp proc}, for any $p \geq 1$,
$$
\PE^{1/p}\bigl[(R_k^{\star})^p\bigr] \leq 3 \ConstC_{0} \sigma_k\,,
$$
where the quantity $\sigma_k$ is defined as
\begin{equation}
\label{eq:sigma_k_definition_correct}
\sigma_k = \gamma^k\bigl\Vert \widehat{V}_{0}^{{\rm {up}}}-V^{{\rm *}}\bigr\Vert_{\Xset} + \left\Vert V^{\pi}-V^{{\rm *}}\right\Vert _{\Xset} +\frac{I_{\mathcal D} +  \mathsf{D}}{\sqrt{M_1}}  + \rho(\Xset_N, \Xset)\,,
\end{equation}
and the constant $\ConstC_{0}$ is given by
\begin{equation}
\label{eq:const_C_0_definition}
\ConstC_{0} = \max\left\{\frac{\gamma L_{\pi}}{1-\gamma} + \frac{R_{\max}}{(1-\gamma)^2}, \frac{(L_{\max} + 2\gamma L_{\pi})\sqrt{2}}{(1-\gamma L_{\psi})(1-\gamma)} + \frac{L_{0}}{1-\gamma}, \frac{\gamma}{1-\gamma}\right\}\,.
\end{equation}
This yields the final bound
\begin{equation}
\label{eq:Var_final_bound}
\Var[\widehat{V}_{k+1}^{{\rm {up}}}(x)] \leq \PE[ \sup_{x \in \Xset} \vert T_k^x \vert^2] \leq 9C_{\Xset, \Aset}\ConstC_0^2 \frac{\sigma_k^{2} + \sigma_k^{2-\varepsilon}/\varepsilon }{M_2} \eqdef \ConstC\frac{\sigma_k^{2} + \sigma_k^{2-\varepsilon}/\varepsilon }{M_2}\,.
\end{equation}
Now the statement follows by the choice $\varepsilon = \log^{-1}(\rme \vee \sigma_k^{-1})$.

\vspace{-2ex}
\subsection{Proof of Proposition~\ref{prop:sigma_k_bound}}
\label{sec:proof_prop_5_1}
The corrected statement of Proposition~\ref{prop:sigma_k_bound} is given below:
\begin{prop}
\label{prop:sigma_k_bound_corrected_1}
Let $|\Xset|, |\Aset| < \infty$, assume A\ref{ass:Pa}, A\ref{ass: r-Vpi}, and $\bigl\Vert \widehat{V}_{0}^{{\rm {up}}} \bigr\Vert_{\Xset} \leq R_{\max}(1-\gamma)^{-1}$. Then for $k$ and $M_1$ large enough, it holds that
\begin{equation}
\label{eq:upper_bound_scaling_correct}
\sigma_k \lesssim \left\Vert V^{\pi}-V^\star \right\Vert _{\Xset}\,.
\end{equation}
The precise bounds for $k$ and $M_1$ are given in  \eqref{eq:k_M_1_bound_appendix}.
\end{prop}
\begin{proof}
Applying \eqref{eq:sigma_k_definition_correct} with $I_{\mathcal D} \lesssim \sqrt{\log{|\Xset||\Aset|}}$, $\mathsf{D} = 1$, we obtain that
\[
\sigma_k \lesssim \gamma^k\bigl\Vert \widehat{V}_{0}^{{\rm {up}}}-V^{{\rm *}}\bigr\Vert _{\Xset} + \left\Vert V^{\pi}-V^{{\rm *}}\right\Vert _{\Xset} +
\frac{\gamma R_{\max}(\sqrt{\log(|\Xset| |\Aset|)} + 1)}{\sqrt{M_1}(1-\gamma)^2}. 
\]
Note that, under assumption A\ref{ass: r-Vpi}, $\bigl\Vert V^\star \bigr\Vert_{\Xset} \leq R_{\max}(1-\gamma)^{-1}$. Hence, the previous bound implies $\sigma_k \lesssim \left\Vert V^{\pi}-V^\star \right\Vert _{\Xset}$, provided that $k$ and $M_1$ are large enough to guarantee 
\begin{align*}
\gamma^{k-1} R_{\max}\leq \left\Vert V^{\pi}-V^\star \right\Vert _{\Xset}, \quad R_{\max}(\sqrt{\log (|\Xset| |\Aset|)}   +  1)M_1^{-1/2}(1 - \gamma)^{-2}\leq \left\Vert V^{\pi}-V^\star \right\Vert _{\Xset}\,.    
\end{align*}
Thus, it is enough to choose
\begin{equation}
\label{eq:k_M_1_bound_appendix}
\begin{split}
&k \geq \log{\left\Vert V^{\pi}-V^\star \right\Vert _{\Xset}}(\log{(1/\gamma)})^{-1}\,, \\
&M_1 \geq R_{\max}^2(\sqrt{\log (|\Xset| |\Aset|)}   +  1)^2((1 - \gamma)^2\left\Vert V^{\pi}-V^\star \right\Vert _{\Xset})^{-2}\,.
\end{split}
\end{equation}
\end{proof}

\subsection{The covering radius of randomly distributed points over a cube}
The following proposition is a particular case of the result \cite[Theorem 2.1]{Reznikov}. We repeat the arguments from that paper and give explicit expressions for the constants. 
\begin{prop}
\label{prop: covering rad}
Let $\Xset = [0,1]^\dx$ and $\mu$ be a uniform distribution on $\Xset$. Suppose that $\Xset_N = \{X_1, \ldots, X_N\}$ is a set of $N$ points independently distributed over $\Xset$ w.r.t. $\mu$. Denote by $\rho(\Xset_N,\Xset) \eqdef \max_{x \in \Xset} \min_{k \in [N]}|x - X_k|$ the covering radius of the set $\Xset_N$ w.r.t. $\Xset$. Then for any $p \geq 1$,
\begin{equation}
\label{cov rad lp}
\EE{\rho^p(\Xsetn, \Xset)}^{1/p} \lesssim \sqrt{\dx} \left( \frac{p \log N}{N} \right)^{1/\dx}.  
\end{equation}
Moreover, for any $\delta \in (0,1)$
\begin{equation}
\label{cov rad prob}
\rho(\Xsetn, \Xset) \lesssim \sqrt{\dx} \left( \frac{\log(1/\delta) \log N}{N} \right)^{1/\dx}  
\end{equation}
holds with probability at least $1 - \delta$.
\end{prop}
\begin{proof}
Let $\mathcal E_n = \mathcal E_n(\Xset)$ be a maximal set of points such that for any $y, z \in \mathcal E_n$ we have $|y - z| \geq 1/n$. Then for any $x \in \Xset$ there exists a point $y \in \mathcal E_n$ such that $|x - y| \le 1/n$. Denote by $B(x, r)$ a ball centred at $x \in \Xset$ of radius $r$ (w.r.t. $| \cdot |$) and 
$$
\Phi(r) = \frac{r^\dx\pi^{\dx/2}}{2^\dx\Gamma(\dx/2+1) }, \, r \in [0, \infty) .
$$
Since for any $x \in \Xset$, $\mu(B(x,r)) \geq \Phi(r)$, 
\begin{equation*}
    1 = \mu(\Xset) \geq \sum_{x \in \mathcal E_n} \mu(B(x, (1/(3n)))) \geq |\mathcal E_n| \Phi(1/(3n)).
\end{equation*}
Hence,
\begin{eqnarray}
\label{eq: cardinality of E_n}
|\mathcal E_n| \le \{\Phi(1/(3n))\}^{-1}.
\end{eqnarray}
Suppose that $\rho(\Xsetn, \Xset) > 2/n$. Then there exists a point $y \in \Xset$ such that $\Xsetn \cap B(y,2/n) = \emptyset$. Choose a point $x \in \mathcal E_n$ such that $|x - y| < 1/n$. Then $B(x, 1/n) \subset B(y, 2/n)$, and so the ball $B(x, 1/n)$ doesn't intersect $\Xsetn$. Hence, $\Xsetn \cap B(x,1/(3n)) = \emptyset$. Therefore,
\begin{multline}
\label{eq: cov rad}
\PP(\rho(\Xsetn, \Xset) > 2/n) \le \PP(\exists x \in \mathcal E_n: \Xsetn \cap B(x,1/(3n)) = \emptyset) \le \\ 
\le|\mathcal E_n| (1 - \Phi(1/(3n))^N 
 \le  |\mathcal E_n| \rme^{-N \Phi(1/(3n)}.
\end{multline}

Let $1/(3n) = \Phi^{-1}(\alpha \log N/N)$ for some $\alpha > 0$ to be chosen later. Then $\Phi(1/(3n)) = \alpha \log N/N$. Inequalities \eqref{eq: cardinality of E_n} and \eqref{eq: cov rad} imply 
\begin{eqnarray}
\PP(\rho(\Xsetn, \Xset) > 2/n) \le \frac{N^{1 - \alpha}}{\alpha \log N}.
\end{eqnarray}
Let us fix any $p \geq 1$. Then 
\begin{eqnarray}
\EE{\rho^p(\Xsetn, \Xset)}^{1/p} \le \frac{2}{n} +  \sqrt{\dx} \left(\frac{ N^{1 - \alpha}}{\alpha \log N} \right)^{1/p} = 6 \Phi^{-1} (\alpha \log N/N) + \sqrt{\dx} \left(\frac{ N^{1 - \alpha}}{\alpha \log N} \right)^{1/p}\,.
\end{eqnarray}
Since
$$
\Phi^{-1}(r) = \frac{2}{\sqrt \pi} \Gamma^{1/\dx}(\dx/2 +1) r^{1/\dx} \le 2 \sqrt{\rme \dx / \pi} (\rme r)^{1/\dx},
$$
we get
$$
\EE{\rho^p(\Xsetn, \Xset)}^{1/p} \le 12 \sqrt{\rme \dx / \pi} \left(\frac{\alpha \rme \log N}{N} \right)^{1/\dx} + \sqrt{\dx} \left(\frac{ N^{1 - \alpha}}{\alpha \log N} \right)^{1/p}.
$$
It remains to take $\alpha = 1 + p/\dx$ to obtain the bound
$$
\EE{\rho^p(\Xsetn, \Xset)}^{1/p} \le 48 \sqrt{\dx} \left(\frac{ p\log N}{N} \right)^{1/\dx}.
$$
Hence, \eqref{cov rad lp} follows. To prove \eqref{cov rad prob} it remains to apply Markov's inequality.
\end{proof}
\vspace{-4ex}
\subsection{Auxiliary results}
\begin{prop}
\label{prop analysis}
For any $\Delta > 0$,
$$
\int_0^\Delta \sqrt{\log(1+1/x)}\rmd x \lesssim \Delta \sqrt{\log(1+1/\Delta)} + \Delta.
$$
\end{prop}
\begin{proof}
Consider first the case $\Delta < 1$. In this case 
\begin{eqnarray*}
\int_0^\Delta \sqrt{\log(1+1/x)}\rmd x &= & \int_0^{\Delta ^{100}/2} \sqrt{\log(1+1/x)}\rmd x + \int_{\Delta ^{100}/2}^\Delta  \sqrt{\log(1+1/x)}\rmd x \\
&\lesssim& \int_0^{\Delta ^{100}/2} x^{-1/2} \rmd x + \int_{\Delta ^{100}/2}^\Delta  \sqrt{\log(1+1/x)}\rmd x \\
& \lesssim & \Delta^{50} + (\Delta  - \Delta ^{100/2}) \sqrt{ \log(1+ 2/\Delta ^{100})} \lesssim \Delta  \sqrt{\log(1+ 1/\Delta )}.
\end{eqnarray*}
Second, if $\Delta > 1$,
\begin{eqnarray*}
\int_0^\Delta \sqrt{\log(1+1/x)}\rmd x = \int_0^{1} \sqrt{\log(1+1/x)}\rmd x + \int_{1}^\Delta  \sqrt{\log(1+1/x)}\rmd x \lesssim \Delta\,. 
\end{eqnarray*}
\end{proof}

\vspace{-4ex}
\section{Experiment Setup}
\label{sec:envs}

\vspace{-2ex}
\subsection{Environment description}
\vspace{-2ex}
\paragraph{Garnet}
The Garnet example is an MDP with randomly generated transition probability kernel \(\MK^a\) with finite state space \(\Xset\) and action space \(\Aset\). This example is described by a tuple $\tup{N_S, N_A, N_B}$. The first two parameters specify the number of states and actions, respectively. The last parameter is responsible for the number of states to which an agent can go from state \(x \in \Xset\) by performing action \(a \in \Aset\). In our case, we used \(N_S=20, N_A=5, N_B=2, \gamma = 0.9\). The reward matrix \(r^a(x)\) is set according to the following principle: for all state-action pairs, the reward is set to be uniformly distributed on [0, 1].

\paragraph{Frozen Lake}
The agent moves in a grid world, where some squares of the lake are walkable, but others lead to the agent falling into the water, so the game restarts. Additionally, the ice is slippery, so the movement direction of the agent is uncertain and only partially depends on the chosen direction. The agent receives 10 points only for finding a path to a goal square, whereas for falling into a hole it does not receive anything. We used the built-in \(4 \times 4\) map and 4 actions for the agent to perform in each state, if available (\texttt{right, left, up} and \texttt{down}). For this experiment, we assume that the reward matrix \(r^a(x)\) is known, and the \(\gamma\)-factor is set to be 0.9.

\paragraph{Chain}
Chain is a finite MDP where the agent can move only to 2 adjacent states, performing 2 actions from each state (\texttt{right} and \texttt{left}). Every chain has two terminal states at the ends. For a transition to the terminal states, the agent receives 10 points and the episode ends, otherwise the reward is equal to +1. Also, there is $p\%$ noise in the system, that is, the agent performs a uniformly random action with probability $p$. For experiments with chains, we set the \(\gamma\)-factor to 0.8, to ensure that Picard iterations converge.

\paragraph{NRoom}
NRoom is a discrete grid-world environment with connected rooms and with one large reward in a single room and small rewards elsewhere. Also, there are traps which lead to terminal states. At each state there are four actions: left, right, up, and down. With a small probability, the chosen action is ignored and a uniformly random action is chosen.

%

\paragraph{CartPole}
CartPole is an example of an environment with a finite action space and infinitely large state space. A reward equal to 1 is gained at every time step until failing or the end of the episode. In fact, CartPole does not have any specific stochastic dynamics, because transitions are deterministic according to actions, so for a non-degenerate case we should add some noise and we apply a normally distributed random variable to the angle. LD (Linear Deterministic) policy can be expressed as I$\{3\cdot\theta + \dot{\theta} > 0\}$, where $\theta$ is an angle between the pole and the normal to the cart.

\paragraph{Acrobot}
The environment consists of two joints, or two links. The torque is applied to the binding between the joints. The state space is six-dimensional, representing two angles (sine and cosine) characterizing the links' positions and the angles' velocities. Each episode starts with small perturbations of the parameters near the resting state with both of the joints in a downward position. At each time step the robot has a reward equal to -1, and it gets 0 in a terminal state, when the boundary has been reached. Also, to make the environment stochastic, a random uniform torque from $-1$ to $1$ is added to the force at each step.

\paragraph{TwinRooms}
TwinRooms is a grid-world environment with continuous state space. It is composed of two rooms separated by a wall, such that $\Xset = ([0,1-\Delta]\cup[1+\Delta, 2]) \times [0,1]$ where $2\Delta = 0.1$ is the width of the wall, as illustrated by Figure~\ref{fig:twin_rooms}. There are four actions: left, right, up, and down, each one resulting in a displacement of 0.1 in the corresponding direction. A two-dimensional Gaussian noise is added to the transitions, and, in each room, there is a single region with non-zero reward. The agent has 0.5 probability of starting in each of the rooms, and the starting position is at the room’s bottom-left corner.

\vspace{-4ex}
\begin{figure}[H]
    \centering
    \includegraphics[scale=0.15]{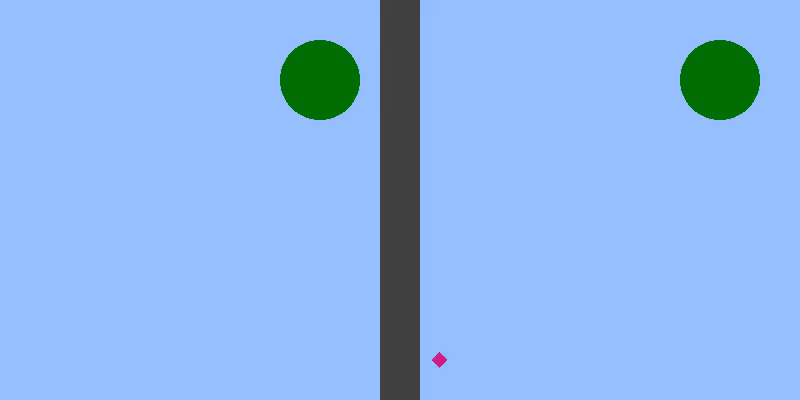} 
    
    \caption{Continuous grid-world environment with two rooms separated by a wall. The circles represent the regions with non-zero rewards.}
    \label{fig:twin_rooms}
\end{figure}

\vspace{-4ex}
\subsection{Experimental setup}
Code is available at \url{https://github.com/levensons/UVIP}. For the sake of completeness, we provide below hyperparameters for the experiments run in Section~\ref{sec:num}.
\vspace{-4ex}
\begin{table}[htb]
\caption{Experimental hyperparameters} 
\begin{center}
\resizebox{0.5\textwidth}{!}{
\begin{tabular}{@{}|lc|c|c|c|c|@{}}\toprule
Environment &\phantom{abc} & $M_1$ & $M_2$ & discount $\gamma$ & $N$ \\ 
\toprule
Garnet      &  & $3000$   & $3000$   & $0.9$ & $-$ \\ \hline
Frozen Lake &  & $1000$   & $1000$   & $0.9$ & $-$\\ \hline
Chain       &  & $1000$   & $1000$   & $0.8$ & $-$\\ \hline
CartPole    &  & $150$ & $150$ & $0.9$ & $1500$\\ \hline
Acrobot     &  & $150$ & $100$ & $0.9$ & $4000$\\
\bottomrule
\end{tabular}}
\label{tab:parameters}
\end{center}
\end{table}
\vspace{-6ex}
\subsection{Auxiliary algorithms}
In Section~\ref{sec:num} we use the Value Iteration algorithm from \cite{szepesvari2010algorithms}, Chapter 1.

\end{document}